\documentclass{article}

\usepackage{arxiv}

\usepackage[utf8]{inputenc} 
\usepackage[T1]{fontenc}    
\usepackage{hyperref}       
\usepackage{url}            
\usepackage{booktabs}       
\usepackage{amsfonts}       
\usepackage{nicefrac}       
\usepackage{microtype}      
\usepackage{lipsum}		
\usepackage{graphicx}
\usepackage{natbib}
\usepackage{doi}

\usepackage{amsmath, amsthm}
\usepackage{tikz}
\usepackage{chemfig}
\usepackage{placeins}
\usetikzlibrary{calc}
\usetikzlibrary{matrix}

\newtheorem{proposition}{Proposition}
\newtheorem{lemma}{Lemma}

\newcommand{\pmsd}[2]{#1\,{\scriptsize\,#2}}

\title{Neural Petri flows for chemical reactions}

\author{Jose Eduardo Escrig Molina \\
	Bioinformatics Group\\
    Bioprocess Engineering Group\\
	Wageningen Unversity\\
	Wageningen, the Netherlands \\
	\texttt{joseeduardo.escrigmolina@wur.nl} \\
	\And
	\href{https://orcid.org/0000-0003-1737-4407}{\includegraphics[scale=0.06]{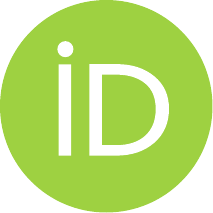}\hspace{1mm}Daniel Probst} \\
	Bioinformatics Group\\
	Wageningen Unversity\\
	Wageningen, the Netherlands \\
	\texttt{daniel.probst@wur.nl} \\
}

\renewcommand{\shorttitle}{Neural Petri flows for chemical reactions}

\hypersetup{
pdftitle={Neural Petri flows for chemical reactions},
pdfauthor={Jose Eduardo Escrig Molina, Daniel Probst},
pdfkeywords={chemistry, neural petri nets, reaction prediction atom mapping, reaction classification, enzymology},
}

\begin{document}
\maketitle

\begin{abstract}
	Petri nets, in which transitions move tokens between places, have been used to describe chemical processes such as reactions. Petri net terms map well to chemistry: Places are the bonds between atoms and the free valence of each atom, a token is a unit of bond order, a transition forms or breaks a bond, the conserved quantities are the valence budgets of the atoms, and the enabling rule, that a transition needs tokens on its input places, is the valence rule. These semantics are not guaranteed by learned models of reactions or neural networks that are built on Petri nets, which use the net as a scaffold for message passing. Here, we ask what architecture remains a Petri net for every value of its weights. We find the answer in the theory, where all semantics of a net share the firing form $m^\prime=m+C\sigma$, locality, as enabling reads only the inputs of a transition, and the enabling rule, and we prove that conservation forces the firing form and that non-negativity forces the enabling rule on local rate laws. What this leaves free is the rate law, which is the propensity of each transition to fire. We introduce Neural Petri Flow, which learns this rate law, or a readout for classification, and hard-wires the rest as parameter-free layers. On what we denote a valence net, atom mapping, reaction classification, and forward prediction become three tasks on one firing vector, which is an unordered list of bond changes. Without training, the minimum firing vector maps 88.8~\% of the curated Golden set against 85.6~\% for RXNMapper, and 88.7 against 77.9~\% of the enzymatic reactions of EnzymeMap. On USPTO-480K, NPF trained on these firing vectors predicts 87.7~\% of the products and 67.4~\% when trained on a 1~\% subset of the training reactions. EC numbers of ECREACT are predicted at the third level for 90.2~\% of reactions, 5.6 points ahead of the best published method. With electrons as tokens, the same token game predicts 90.5~\% of the elementary steps of FlowER first, ahead of the published baseline, and every top-1 prediction is a valid molecule without a filter.
\end{abstract}


\section{Introduction}
Petri nets were devised in order to describe chemical processes~\citep{Petri2008}. While Petri nets became a standard tool for concurrency, chemistry kept them to describe reaction networks, where places are species~\citep{Koch2010,Baez2019}. We take this approach down to chemical bonds of a single reaction, and ask what kind of neural network it takes for its outputs to obey the rules of that net for every value of its weights, as the rules are computed instead of learned. Learning only fills in what the rules leave open, such as which bonds change.

\paragraph{A reaction as a net.} Given the heavy atoms, which are all atoms other than hydrogen, of the precursors of a chemical reaction, we define a bond place $B_{ij}$ for every pair $(i,j)$ of atoms and a slack place $S_i$ for every atom $i$. Tokens on $B_{ij}$ are units of bond order, where a single bond is one token, a double bond is two tokens, and an aromatic bond one and a half. Tokens on $S_i$ are units of valence of atom $i$ not spent on bonds with other heavy atoms, namely, its hydrogens, its negative charge, and a capacity to become cationic or hypervalent, one token for nitrogen and none for carbon. A marking is the vector of all token counts or, in chemistry terms, a set of molecules. A reaction is then a pair of markings, $m_A$ and $m_B$ for the reactants and the products, respectively. The two are written on different atoms, and an atom map, which assigns every product atom to the matching precursor atom, makes them markings of one net. Transitions $t$ are elementary bond changes, where, for every atom pair $(i,j)$ and every amount $\delta$, a transition takes $\delta$ tokens from each of $S_i$ and $S_j$ and puts $\delta$ tokens on $B_{ij}$, forming a bond of order $\delta$. The reverse operation breaks the bond. Firing a transition is the realisation of that bond change, and the firing-count vector $\sigma$, firing vector for short, counts the firings of every transition, describing a reaction as an unordered list of bond changes. The state equation $m_B=m_A+C\sigma$ reads products = precursors + bond changes, where the matrix $C$ records which places each transition empties and fills. The $P$-invariants are the quantities no firing can change, namely, the valence budgets of the atoms, each the bond orders plus its free valence of an atom. The enabling rule that states that a transition may only fire when its input places hold enough tokens, is the valence rule. For example, a carbon with four bonds to other heavy atoms has to break one before it can form another. This net, the valence net of the reaction, has no parameters, as it is fixed by the precursor atoms and the rules of valence.

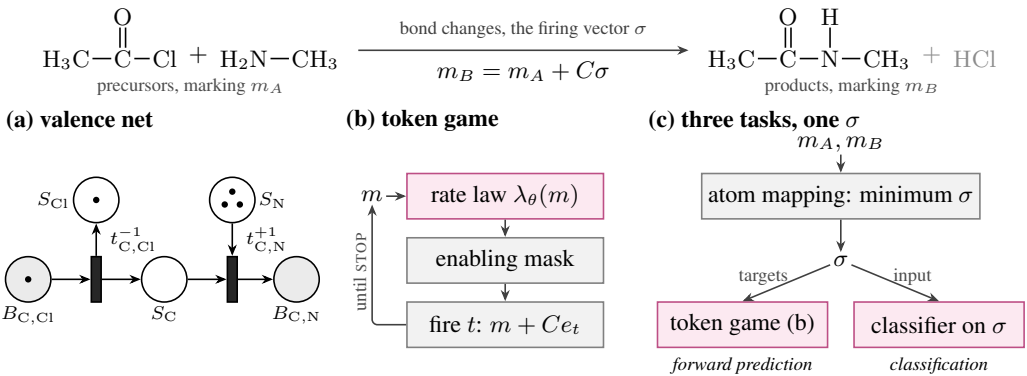
\begin{figure}[t]
    \centering
    \setchemfig{atom sep=1.75em, bond offset=1pt, double bond sep=2.1pt, bond style={line width=0.6pt}}%
\begin{tikzpicture}[
    font=\footnotesize, >={Stealth[length=1.5mm, width=1.2mm]},
    place/.style={circle, draw, line width=0.6pt, minimum size=6mm, inner sep=0pt},
    bondplace/.style={place, fill=black!8},
    trans/.style={rectangle, draw, line width=0.6pt, fill=black!85, minimum width=1.3mm, minimum height=6mm,
                  inner sep=0pt},
    tok/.style={circle, fill=black, inner sep=0.85pt},
    arc/.style={->, line width=0.6pt},
    fixed/.style={rectangle, draw=black!50, fill=black!5, align=center, minimum height=6mm,
                  inner sep=2.4pt, line width=0.6pt},
    learned/.style={fixed, draw=magenta!75!black, fill=magenta!10},
    io/.style={inner sep=1pt, align=center},
    flow/.style={->, line width=0.6pt, black!75},
    head/.style={font=\footnotesize\bfseries, anchor=west, inner sep=0pt},
    note/.style={font=\scriptsize, text=black!70, inner sep=0pt, align=left},
    task/.style={font=\scriptsize\itshape, inner sep=0pt}]

  \node[inner sep=0pt] (pre) at (2.45, 5.13) {\chemfig{H_3C-C(=[:90]O)-Cl}\enspace\raisebox{0.5ex}{$+$}\enspace\chemfig{H_2N-CH_3}};
  \node[inner sep=0pt] (pro) at (11.2, 5.13)
       {\chemfig{H_3C-C(=[:90]O)-N(-[:90]H)-CH_3}\enspace\raisebox{0.5ex}{\textcolor{black!45}{$+$}}\enspace\textcolor{black!45}{\chemfig{HCl}}};
  \draw[flow] ($(pre.east)+(0.25,-0.15)$) --
       node[above=1pt, font=\scriptsize] {bond changes, the firing vector $\sigma$} ($(pro.west)+(-0.25,-0.15)$);
  \node[note, anchor=north] at ($(pre.south)+(0,-0.1)$) {precursors, marking $m_A$};
  \node[note, anchor=north] at ($(pro.south)+(0,-0.1)$) {products, marking $m_B$};
  \node[font=\small] at (6.85, 4.7) {$m_B = m_A + C\sigma$};

  \node[head] at (0, 4.05) {(a) valence net};

  \node[bondplace, label={[inner sep=1pt, font=\scriptsize]below:$B_{\mathrm{C,Cl}}$}] (bccl) at (0.3, 1.95) {};
  \node[trans] (tm) at (1.2, 1.95) {};
  \node[place, label={[inner sep=1pt, font=\scriptsize]below:$S_{\mathrm{C}}$}] (sc) at (2.1, 1.95) {};
  \node[trans] (tp) at (3.0, 1.95) {};
  \node[bondplace, label={[inner sep=1pt, font=\scriptsize]below:$B_{\mathrm{C,N}}$}] (bcn) at (3.85, 1.95) {};
  \node[place, label={[inner sep=1pt, font=\scriptsize]left:$S_{\mathrm{Cl}}$}] (scl) at (1.2, 3.0) {};
  \node[place, label={[inner sep=1pt, font=\scriptsize]right:$S_{\mathrm{N}}$}] (sn) at (3.0, 3.0) {};

  \draw[arc] (bccl) -- (tm);
  \draw[arc] (tm) -- (sc);
  \draw[arc] (tm) -- (scl);
  \draw[arc] (sc) -- (tp);
  \draw[arc] (sn) -- (tp);
  \draw[arc] (tp) -- (bcn);

  \node[tok] at (bccl) {};
  \node[tok] at (scl) {};
  \node[tok] at ($(sn)+(-0.1,-0.06)$) {};
  \node[tok] at ($(sn)+(0.1,-0.06)$) {};
  \node[tok] at ($(sn)+(0,0.11)$) {};

  \node[anchor=west, inner sep=0pt, font=\scriptsize] at ($(tm.east)+(0.1,0.52)$) {$t^{-1}_{\mathrm{C,Cl}}$};
  \node[anchor=west, inner sep=0pt, font=\scriptsize] at ($(tp.east)+(0.1,0.52)$) {$t^{+1}_{\mathrm{C,N}}$};

  \node[head] at (4.5, 4.05) {(b) token game};

  \node[io] (m) at (4.85, 3.05) {$m$};
  \node[learned, anchor=west, minimum width=2.6cm] (rate) at (5.3, 3.05) {rate law $\lambda_\theta(m)$};
  \node[fixed, anchor=west, minimum width=2.6cm] (mask) at (5.3, 2.2) {enabling mask};
  \node[fixed, anchor=west, minimum width=2.6cm] (fire) at (5.3, 1.35) {fire $t$: $m + C e_t$};
  \draw[flow] (m) -- (rate);
  \draw[flow] (rate) -- (mask);
  \draw[flow] (mask) -- (fire);
  \draw[flow, rounded corners=2pt] (fire.west) -| (m.south);
  \node[note, rotate=90, anchor=south] at (4.78, 2.05) {until \textsc{stop}};

  \node[head] at (8.5, 4.05) {(c) three tasks, one $\sigma$};

  \node[fixed] (map) at (11.05, 3.05) {atom mapping: minimum $\sigma$};
  \node[io, anchor=south] (min) at ($(map.north)+(0,0.24)$) {$m_A, m_B$};
  \draw[flow] (min) -- (map);

  \node[io, font=\small] (sigma) at (11.05, 2.2) {$\sigma$};
  \draw[flow] (map) -- (sigma);

  \node[learned, minimum width=2.25cm] (tg) at (9.75, 1.35) {token game (b)};
  \node[learned, minimum width=2.25cm] (ro) at (12.35, 1.35) {classifier on $\sigma$};
  \draw[flow] (sigma) -- node[above left=-1pt, note] {targets} (tg.north);
  \draw[flow] (sigma) -- node[above right=-1pt, note] {input} (ro.north);

  \node[task, anchor=north] at ($(tg.south)+(0,-0.12)$) {forward prediction};
  \node[task, anchor=north] at ($(ro.south)+(0,-0.12)$) {classification};

  \pgfresetboundingbox
  \useasboundingbox (-0.06, 0.7) rectangle ++(13.6, 4.8);
\end{tikzpicture}%
    \caption{(a) A reaction as two markings of its valence net. (b) Only the rate law (magenta) is learned. (c) The minimum firing vector $\sigma$ maps atoms, trains the token game and feeds the classifier.}
    \label{fig:framework}
\end{figure}

Sequence models can write strings that are not a valid molecule, and electron-flow and graph-edit models can predict bond changes that are not guaranteed to respect valence for every output. Stated in net terms, the outputs of these models are not markings (states) that any firing sequence can reach. Current Petri graph neural networks pass messages between places and transitions, using the net as a computation graph, but the states they update are not markings~\citep{Hu2020,AdemovicTahirovic2025}. We therefore ask whether a model of reactions can, for every value of its weights, guarantee products with valid valences and bond changes accounting for the observed change exactly, and whether this guarantee increases or decreases accuracy.

\paragraph{Approach.} Discrete, continuous, and stochastic nets, the three semantics of a net, only differ in their rate law, which is the propensity of a transition to fire given the tokens on its input places, and share the firing form, the enabling rule, and locality. Therefore, Neural Petri Flow (NPF) only learns the rate law, or a readout for classification, and computes the rest without trainable parameters. Atom mapping solves the state equation for the firing vector that forms or breaks the fewest bonds without learned weights, and gives the forward model its training target. For forward prediction, NPF plays a token game in which it fires one bond change at a time with the learned rate law choosing each move (Figure~\ref{fig:framework}). For chemistry, the rate law reads the whole marking and as the token game fires one bond change at a time, non-negativity still requires the enabling rule, which is imposed as a mask (Section~\ref{sec:neural_petri_flow}). With electrons as tokens, the same token game predicts the elementary steps of reaction mechanisms (Section~\ref{sec:chemical_reactions}).

\paragraph{More than a Petri net in name only.} To avoid introducing a generic neural network that is a Petri net in name only, we support every Petri component of the valence net that carries a guarantee with a proof, a reason no neural network without the Petri net can do the same, a numerical check, and every component except conservation with a comparison to a generic network of equal size. For the valence rule, we prove that, for every value of its weights, the token game never gives a carbon with a valence budget of four a fifth bond (Proposition~\ref{prop:validity}), and that the rule removes orders of bond changes but never a product within the valence budgets (Lemma~\ref{lem:realisability}. With untrained weights, no marking that the token game visits breaks the rule by more than half a token (Appendix~\ref{sec:app_load_bearing}). Graph-edit and electron-flow models also change one bond or electron pair at a time, however, most fix or learn an order of edits~\citep{do2019,Sacha2021}, learn the valence rule, as FlowER does~\citep{Joung2025}, or check only that the giving site has electrons, and not that the receiving atom has room, as in MAELLE~\citep{Nguyen2026}. In NPF, the valence rule holds at every step for every weight, and the state equation makes the order irrelevant.

\paragraph{Contributions.}

\begin{itemize}
\item \textbf{Neural networks that remain Petri nets.} Neural Petri Flow only learns what the theory of the net leaves open, namely, one rate law that is shared across nets with conflict or, for classification tasks, a readout of the firing vector, and computes the firing form, the enabling rule, and the state equation as layers without trainable parameters. The conservation of every $P$-invariant forces the firing form, and keeping markings non-negative forces the enabling rule on any continuous, local rate law shared across nets. For every value of the weights, every marking visited by the token game is reachable in the valence net once every atom holds half a token of extra free valence for aromatic bonds of order $\frac{3}{2}$.
\item \textbf{The valence net, a single object for three tasks.} Atom mapping, reaction classification, and reaction forward prediction become three tasks on one firing vector, which is the unordered list of bond changes. Mapping infers it, classification reads a label off it, and the forward model predicts it. Like other atom mappers of minimum chemical distance~\citep{Jochum1980,First2011,Mann2014,Ali2025}, our mapper minimises bond changes, without learned weights, and it maps reactions with a missing reactant or by-product (Appendix~\ref{sec:minimum_firing_vector}) and proves the map optimal for 100~\% of a curated reaction set. Its tied optima are training targets of the forward model, which therefore needs neither the maps of a learned model nor an order of edits.
\item \textbf{Mechanisms as nets of electrons.} Taking electrons as tokens turns the elementary steps of a mechanism into firings of an electron net, in which fishhook and curly arrows are the transitions, the number of electrons is a $P$-invariant, and the octet rule is the enabling rule. With the rate law and token game of the forward model, it predicts $90.5$~\% of the elementary steps of FlowER first and recovers $94.5$~\% of the pathways, ahead of the published baselines, and every end marking is a Lewis structure by construction.
\end{itemize}

\section{Related work}
\label{sec:related-work}

\paragraph{Petri nets and learning.} \citet{AbdulHameed2025} introduced a max-plus network learning the timing of timed event graphs, which are conflict-free nets. Trainable mass-action reaction networks as introduced by \citet{Nagipogu2025} and \citet{Dack2026} and a neural Petri net computed by the state equation by \citet{MedinaMarin2023} are also nets for every value of their parameters, each for one system. Message passing along Pre and Post was introduced as a Petri net convolution for scheduling by \citet{Hu2020}, and PGNN by \citet{AdemovicTahirovic2025}, which we use as a baseline, passes messages between the places and transitions of attributed nets.

\paragraph{Conservation and structure in neural networks.} Hybrid models where a neural network predicts rates inside first-principle balances are the closest to the token-game layer from Appendix~\ref{sec:forward_semantics}~\citep{Psichogios1992}. \citet{Sturm2022conservation} discuss a fixed stoichiometric matrix on learned fluxes of atoms between molecules in atmospheric photochemistry, with kinetic models pass learned rates through such a layer~\citep{Kircher2024,Fonseca2026}. Chemical reaction neural networks hard-wire mass action and learn the associated stoichiometry~\citep{Ji2021,Dppel2024}. These models realise the sufficiency direction of Proposition~\ref{prop:conservation}. 

\paragraph{Reactions.} Petri nets are used to model reaction networks with species as places~\citep{Koch2010,Baez2019,Andersen2025}, while reversing Petri nets enable the symbolic modelling of covalent bonds~\citep{Melgratti2022}. We use a net a level below the species, for the bonds in a single reaction, where the state equation on the bond places is the bond--electron matrix equation introduced by~\citet{Dugundji1973}. Our firing vector is, up to hydrogens and charges, the imaginary transition state or the condensed graph of reaction, counted as firings~\citep{Fujita1986,Varnek2005}. Atom mappers are learned~\citep{Schwaller2021Extraction,Chen2024Precise,Astero2025} or compute the mapping of minimum chemical distances~\citep{Jochum1980}, solved exactly as integer or constraint programs~\citep{First2011,Latendresse2012,Mann2014,Ali2025}. As for reaction classifiers, they use fingerprints~\citep{Schneider2015,Probst2022reaction} or learned representations~\citep{Schwaller2021,Heid2021machine}, while our readout is a learned difference fingerprint with a proof that what is left unchanged by a reaction cancels. Finally, forward reaction predictors generate SMILES~\citep{Schwaller2019,Tetko2020,Irwin2022,Tu2022,Zhong2022}, apply templates~\citep{Chen2022}, predict graph edits~\citep{Jin2017,Coley2019,do2019,Sacha2021,Bradshaw2018}, or redistribute electrons~\citep{Bi2021,Meng2023,Joung2025,Cheng2025Predicting}. While the latter conserve electrons, none of them enforces valence on every discrete output. Recently, ~\citet{Nguyen2026} introduced MAELLE, which learns a Markov chain of electron moves from unordered move sets. It is a stochastic net in our terms, with enabling rules on the source only, no valence invariant, and products ranked by their sampling frequency, training on the atom maps of RXNMapper, where our targets come from the net.

\section{Preliminaries}
\label{sec:preliminaries}

\paragraph{Nets.} A net is defined as $N=(P,T,\mathrm{Pre},\mathrm{Post})$, where places $P$ and transitions $T$ are finite disjoint sets and $\mathrm{Pre},\mathrm{Post}:P\times T\rightarrow\mathbb{R}_{\geq0}$ are arc weights. For the input and output places of transition $t$, we write $^{\bullet}t=\{p:\mathrm{Pre}(p,t)>0\}$ and $t^{\bullet}=\{p:\mathrm{Post}(p,t)>0\}$, respectively. The incidence matrix of the net is $C=\mathrm{Post}-\mathrm{Pre}\in\mathbb{R}^{P\times T}$, where $C_t$ are the columns.

\paragraph{Markings, enabling, and firing.} A marking $m\in\mathbb{R}_{\geq0}^P$ assigns $m_p$ tokens to place $p$. $\mathrm{enabl}(t,m):=\min_{p\in ^{\bullet}t}m_p/\mathrm{Pre}(p,t)$ is the enabling degree of $t$ at $m$, and $\infty$ if $^\bullet t=\emptyset$. In the discrete case, $t$ is enabled if $\mathrm{enabl}(t,m)\geq1$ and fires in amount $\alpha=1$ \citep{David2010}, and in the continuous case in any amount $0<\alpha\leq\mathrm{enabl}(t,m)$. A firing yields $m^\prime=m+\alpha C_t$, written as $m\xrightarrow{\alpha t}m^\prime$. A firing sequence $s=\alpha_1 t_1\dots \alpha_n t_n$ is enabled at $m_0$ if every step can fire, in its amount, at the marking reached so far, and we call the markings it reaches reachable from $m_0$. The firing-count vector $\sigma(s)\in\mathbb{R}_{\geq0}^T$ has the entries $\sigma_t=\sum_{k:t_k=t}\alpha_k$.

\paragraph{State equation and invariants.} The state equation is given by summing the firing rule along a sequence
\begin{equation}
\label{eq:state}
    m_0\xrightarrow{s}m\Rightarrow m=m_0+C\sigma(s).
\end{equation}
The converse of this fails in general, as a non-negative solution $\sigma$ of $m=m_0+C\sigma$ does not need to be the count vector of any enabled sequence~\citep{Murata1989petri}, however, on the valence net it holds for every firing vector that changes each bond place at most once and reaches a non-negative marking (Lemma~\ref{lem:realisability}). Two sequences enabled at $m_0$, with the same count vector, reach the same marking by Equation~\ref{eq:state}, and we call them count-equivalent. A vector $x\in\mathbb{R}^P$ with $x^\top C=0$ is a P-invariant. By \eqref{eq:state}, $x^\top m=x^\top m_0$ for every $m$ of the form $m_0+C\sigma$, specifically for every reachable marking. A vector $y\in\mathbb{R}^T$ with $Cy=0$ is a T-invariant. By \eqref{eq:state}, an enabled firing sequence returns the net to its starting marking exactly when its count vector is a T-invariant. In chemistry terms, different orders of the same bond changes are count-equivalent sequences.

\paragraph{Timed semantics.} The stochastic net introduced by \citet{Molloy1982} is a continuous-time Markov chain on $\mathbb{N}^P$ that jumps from $m$ to $m+C_t$ at a rate $\lambda_t(m)$, where $\lambda_t(m)=0$ unless $t$ is enabled at $m$. $N_t(\tau)$ counts the firings of $t$ up to time $\tau$, so $m(\tau)=m(0)+CN(\tau)$, conserving $P$-invariants pathwise, which ensures that every atom keeps its valence budget after every single bond change and on every sampled route. The rate law $\lambda$ is local if $\lambda_t$ depends on $m$ through $(m_p)_{p\in^\bullet t}$ only. The rate law represents the chemistry, namely, how likely each bond change is given the molecules present.

\paragraph{Two learning problems.} The net, the static attributes $e_p$ and $a_t$ of places and transitions, respectively, and the markings are observed, while the rate law is not. The forward problem to predict $m(\tau+\Delta)$ from $m(\tau)$, which is, in chemistry terms, predicting the products from the precursors. The inverse problem is the prediction of the bond changes, or firing counts $\sigma^*=N(\tau_B)-N(\tau_A)$ given the precursors $A$ and the products $B$.

In reactions, the products are recorded on their own atoms. Then, an atom map, an injection $\pi$ of the product atoms onto the precursor atoms of the matching element, places them on the precursor places as $P_\pi m_B$. The inverse problem is atom mapping, that is, to find $\pi$ and $\sigma^*$ with $P_\pi m_B=m_A+C\sigma^*$ on the bond places. As we show in Appendix~\ref{sec:markings_identify}, the markings fix neither $\pi$ nor the part of $\sigma^*$ in $\ker C$. Classification reads a label off $\sigma^*$.

\section{Neural Petri flow}
\label{sec:neural_petri_flow}
The semantics introduced in Section~\ref{sec:preliminaries} share state changes of the form $C\sigma$ with $\sigma\geq 0$, which we call the firing form, and locality, as enabling reads only the input places of a transition. We hard-wire these while learning the rate law, the single object the theory leaves free. For classification, a readout is learned instead (Section~\ref{sec:state_equation_readout}). Appendix~\ref{sec:general_nets} tests the layers with a local rate law on measured metabolic fluxes.

\subsection{Enforcing the firing form and the enabling rule}
\label{sec:enforcing}
We show that the firing form and the enabling rule are forced, rather than chosen. If a marking is updated, every $P$-invariant at every marking is kept iff it is a combination of the columns of $C,u(m)=Cv(m)$ (Proposition~\ref{prop:conservation}). This means that a model exactly keeps every valence budget for all weights whenever it changes the marking by bond changes, while a generic readout with an output bias would conserve only for a set of weights of measure zero. The enabling rule follows from non-negativity. If a rate law is local and shared across nets, continuous, and non-negative, it keeps every marking non-negative only if it is zero whenever an input place is empty (Proposition~\ref{prop:enabling_rule}). In chemistry terms, a negative slack place is an atom with more bonds than allowed by the valence rule. Proposition~\ref{prop:conservation} extends Theorem 3.1 of \citet{Horie2024}, which states that message passing on a graph conserves only a total with antisymmetric fluxes, to all $P$-invariants of a net, and it keeps a learned update in the stoichiometric subspace of reaction network theory~\citep{Feinberg2019}. \citet{Angeli2007} assume a zero rate when a reactant is absent, and the absence of negative cross-effects characterises mass-action kinetics~\citep{Craciun2020}. Proposition~\ref{prop:enabling_rule} derives this zero rate for a rate law that is shared across nets, given that learned weights are the same on every net. 

However, for reactions the rate law is not local, as the reactivity of a bond depends on its molecule and on the available reagents, so Proposition~\ref{prop:cancellation} does not apply and we impose the enabling rule as a mask, which sets the rate of every transition that is not enabled to zero (Figure~\ref{fig:tokengame}b). Such masks are standard for valence in molecule generation, where bonds are only added~\citep{Liu2018}, and for enabled transitions in reinforcement learning on nets~\citep{Sachweh2024,Lassoued2024}. This is the only way to stay valid for all weights for one firing at a time, as a not enabled bond-forming transition takes more tokens from a slack place than it holds (Appendix~\ref{sec:proofs}).

\subsection{The state-equation readout}
\label{sec:state_equation_readout}

In order to classify a reaction, NPF reads the change described by the state equation without the atom map that would place both markings on a single net. A local encoder, which is message passing along bonds for $K$ rounds, assigns every atom and bond of each side a learned descriptor $\psi$, and a small multilayer perceptron maps the difference $r=\sum_{i\in V_B}\psi_B(i)-\sum_{j\in V_A}\psi_A(j)$ to the class. For all weights and every atom map, an atom whose surroundings within $R$ bonds are the same on both sides cancels exactly against its partner, with $R=K$ for atom terms and $R=K+1$ for bond terms (Proposition~\ref{prop:cancellation}). What remains are the atoms whose bonds $C\sigma(\pi)$ or attributes change, those within $R$ bonds of them, and the atoms that leave. Therefore, the readout $r$ depends only on these atoms and is computed without $\pi$ or $\sigma(\pi)$. A readout that pools each side, e.g. $\rho\left(\sum_B\psi,\sum_A\psi\right)$, is, in general, not invariant to a far decoration (Appendix~\ref{sec:proofs}).

Solvents and catalysts without a partner do not cancel, so a participation gate $w_M\in(0,1)$ for every precursor is trained on the maps of the mapper and moves them to a separate vector. Proposition~\ref{prop:cancellation} therefore only holds for the pure readout $w\equiv1$. We introduce a second classifier that reads the firing vector $\hat{\sigma}$ of the mapper instead, with one learned term per transition that fires.

\section{Reactions as firing sequences of valence and electron nets}
\label{sec:chemical_reactions}

\paragraph{The valence net}
A bond place $B_{ij}$ holds the bond order $b_{ij}\in\{0,1,\frac{3}{2},2,3\}$, and a slack place $S_i$ holds the free valence $s_i=h_i+\max(-q_i,0)+c_i$ of atom $i$, from its hydrogens $h_i$ and negative charge $q_i$, and a capacity of $c_i$ of 1 for N, O, Cl, Br, and I, 2 for S and P, and 0 for all other elements. For every pair of atoms $i<j$ and every $\delta\in\{\frac{1}{2},1,\frac{3}{2},2,3\}$, which is the difference between two bond orders, a transition increases $b_{ij}$ by $\delta$, while its reverse decreases it: 

\begin{equation}
    \label{eq:valence}
    t_{ij}^{+\delta}: \delta S_i + \delta S_j \rightarrow \delta B_{ij}, \quad t_{ij}^{-\delta}: \delta B_{ij} \rightarrow \delta S_i + \delta S_j.
\end{equation}

The $P$-invariants are spanned by the $|V|$ valence budgets $x^{(i)\top}m=\sum_j b_{ij}+s_i$. While the net conserves heavy atoms and valence budgets, it does not conserve hydrogen or charge, which decoding reads back from the free valence Appendix~\ref{sec:implementation_details}.

\paragraph{Atom mapping by the minimum firing vector.} An atom map $\pi$ seats the product marking on the precursor places, while the state equation $P_\pi m_B-m_A=C\sigma(\pi)$ fixes the firing vector on every bond place with a seated end (Figure~\ref{fig:firingvector}). We map by the cheapest firing vector, based on the principle of minimum chemical distance~\citep{Jochum1980}, under a lexicographic cost without continuous parameters. The first level counts the bonds that are formed and broken, with a hydrogen on carbon as a bond of its own, and a hydrogen on a heteroatom, which exchanges with the medium, not counted. The second level then counts the changes of bond order on kept bonds and of formal charge. A branch and bound over the seatings that is bounded by linear assignments on the places of the net then finds the minimum and proves it optimal for $100.0$ \% of the Golden set within 60 s, in a median time of $0.02$ s per reaction. Mappings that tie on both levels are ranked by a third level of three counts, chosen on 200 reactions from the Golden set (Appendix~\ref{sec:minimum_firing_vector}). The mapper runs once per reaction before training and its firing vectors feed the token game and the second classifier.

\paragraph{Forward prediction by the token game.} The product is the marking that an enabled firing sequence reaches from $m_A$. NPF plays the net one bond change at a time (Figure~\ref{fig:tokengame}). From the current marking $m$, the next event is:

\begin{equation}
    \label{eq:token_game}
    \operatorname{Pr}(t \mid m) = \frac{\lambda_{\theta,t}(m) [t \text{ enabled at } m]}{\lambda_{\theta,\mathrm{STOP}}(m) + \sum_{t'} \lambda_{\theta,t'}(m) [t' \text{ enabled at } m]},
\end{equation}

the embedded jump chain of the stochastic net in Section~\ref{sec:preliminaries}, its sequence of markings without the waiting times, where STOP ends the game. A transition is enabled if allowed by the valence rule with half a token of tolerance on slack places (aromaticity), and each bond place is re-typed at most once. The rate law, which is a message-passing network with attention over the whole marking (Appendix~\ref{sec:implementation_details}), ranks bond changes and is fitted to products that carry no times, so it is not a kinetic rate law. The targets are the firing vectors $\sigma$ of the cheapest mappings that decode to the product, a set $\Sigma^*$, with $|\Sigma^*|\leq8$. Because the loss sums over the whole set, the model never has to choose among equally cheap mappings. Furthermore, no order is needed as count-equivalent sequences reach the same marking. In training, we drop the model at a random point along a correct path, and reward it for putting probability on any correct next step. Beam search then keys partial predictions by firing vector, thereby merging count-equivalent sequences and adding their probabilities (Figure~\ref{fig:tokengame}b). The resulting final marking still holds every precursor atom and is decoded into molecules by conservation rules, so by-products omitted in the input data set, come with the prediction. Note that the order of firings is valence bookkeeping and not a mechanism. Products carry no stereochemistry.

\paragraph{The electron net} For mechanisms, we introduce a net whose tokens are electrons. Every atom has an atom place holding its non-bonding electrons, and every pair of atoms has a bond place holding the electrons they share, two for a single bond. A transition moves one electron, a fish-hook arrow, or a pair of electrons, a curly arrow, from an atom place into a bond place, or removes one from a bond place into an atom place. The total number of electrons is a $P$-invariant and no transition adds an atom, so, each step conserves atoms, protons, and electrons by construction. 

In this case, the enabling rule is the octet rule, by which an electron may only move onto an atom with shell room available. The shell capacities and the possible formal charges are read per element from the training steps (Appendix~\ref{sec:implementation_details}), and the guarantees hold relative to them. The STOP is enabled only when every formal charge is whole and within its window and every bond place holds a whole number of pairs, so every end marking is a Lewis structure. Restricting the transitions to pairs gives the arrow net as a subnet whose firings are the curly arrows of arrow pushing, and cannot express steps that move single electrons. The rate law and the token game are the same as in the forward model, a step fires at most 16 arrows and never undoes one, and the products are the end marking. Molecules enter both nets as Kekulé structures (Appendix~\ref{sec:implementation_details}).

\section{Experiments and Discussion}
\label{sec:experiments}

We test each task on standard benchmarks (Tables~\ref{tab:mapping} to~\ref{tab:mechanism_pathways}), described in Appendix~\ref{sec:implementation_details}. For every model, we report the mean and sample standard deviation over three seeds unless stated otherwise, and baselines as reported by their authors on the same splits.


\subsection{Reactions}

\paragraph{Atom mapping.} On the Golden set, scored as by \citet{Lin2021} on the $1\,760$ reactions with a curated map, the minimum firing vector reproduces the curated map, up to the symmetry of the molecules, for $88.81$~\%, and RXNMapper, a transformer pretrained on 2.8M reactions for $85.57$~\% ($p<0.001$), and $88.53$ against $85.58$~\% on the $1\,560$ reactions not used to choose the cost (see Appendix~\ref{sec:minimum_firing_vector}). On SynRXN (Table~\ref{tab:mapping}), which compares reaction centres, our mapper is ahead of RXNMapper on Golden ($p=0.003$) and its subset NatComm ($p=0.01$). It is level with the graph-based learned mappers by \citet{Chen2024Precise} and \citet{Nugmanov2022} on Golden ($p\geq0.51$), non-significantly behind them on NatComm. It is behind all three learned mappers on USPTO-3k ($p<0.001$). SynRXN Golden and NatComm contain 121 and 24 of the 200 reactions used to choose the cost, and without them our mapper scores $89.37$ and $90.36$~\%, still ahead of RXNMapper ($p=0.006$ and $p=0.03$). On all three organic sets, our mapper is ahead of the rule-based RDT by \citet{Rahman2016}. On Recon3D it is ahead of the learned mappers, by at least $7.1$ points ($p<0.001$) and RDT by 2.9 ($p=0.02$). On $E.~coli$, RDT is ahead, although not significantly ($p=0.06$), and we match RXNMapper. To break the tie on \textit{E.~coli}, we evaluated both on EnzymeMap \citep{Heid2023EnzymeMap}, where we establish a lead over RXNMapper with $88.70$ against $77.92$~\% Table~\ref{tab:enzymemap_mapping}.



\begin{table}[h]
\centering
\footnotesize
\setlength{\tabcolsep}{3pt}
\setlength{\abovecaptionskip}{2pt}
\setlength{\belowcaptionskip}{2pt}
\renewcommand{\arraystretch}{0.95}
\caption{Atom mapping on SynRXN~\citep{Phan2026}. Best in \textbf{bold}, second \underline{underlined}. Statistical tests and detailed numbers can be found in Table~\ref{tab:mcnemar}.}
\label{tab:mapping}
\label{tab:mapping_enzyme} 
\begin{tabular}{@{}llrrrrr@{}}
\toprule
method & learned & Golden & NatComm & USPTO-3k & Recon3D & E.\ coli \\
\midrule
RDT & rules & 82.54 & 84.11 & 90.87 & \underline{54.97} & \textbf{78.02} \\
RXNMapper & yes & 87.43 & 87.58 & 93.53 & 48.69 & 72.53 \\
LocalMapper & yes & 89.08 & \underline{92.67} & \textbf{97.77} & 50.79 & 69.96 \\
GraphormerMapper & yes & \underline{89.59} & \textbf{92.87} & \underline{95.10} & 34.82 & 42.12 \\
NPF, minimum firing vector & no & \textbf{89.70} & 90.63 & 92.13 & \textbf{57.85} & \underline{72.89} \\
\bottomrule
\end{tabular}
\end{table}

\paragraph{Reaction classification.}On Schneider50k, the classifier that reads the firing vector of the mapper reaches $98.80\pm0.04$~\% accuracy (Table~\ref{tab:classification}), ahead of the generic network on the same firing vector ($98.64\pm0.09$~\%, $p=0.01$) and of DRFP with a perceptron ($95.52\pm0.04$~\%). Reading the
two sides of the reaction without any map at test time, the state-equation readout reaches $97.68\pm0.09$~\%, level with the generic readout ($97.74\pm0.01$~\%, $p=0.21$), and with all labels the cancellation of Proposition~\ref{prop:cancellation} does not change accuracy, while the firing vector adds $1.1$ points, so the map of the net carries information that the two sides alone do not expose. As rules on the reaction centre assigned the classes of Schneider50k~\citep{Schneider2015}, the EC numbers are the harder test.

\begin{table}[h]
\centering
\footnotesize
\setlength{\tabcolsep}{3pt}
\setlength{\abovecaptionskip}{2pt}
\setlength{\belowcaptionskip}{2pt}
\renewcommand{\arraystretch}{0.95}
\newcommand{\mpm}[2]{$#1 \pm #2$}
\newcommand{\bmpm}[2]{{\boldmath$#1 \pm #2$}}
\newcommand{\smpm}[2]{\begin{tabular}[t]{@{}r@{}}$#1$\\[-1pt]{\scriptsize$\pm #2$}\end{tabular}}
\newcommand{\sbmpm}[2]{\begin{tabular}[t]{@{}r@{}}{\boldmath$#1$}\\[-1pt]{\scriptsize$\pm #2$}\end{tabular}}
\newcommand{\sumpm}[2]{\begin{tabular}[t]{@{}r@{}}\underline{$#1$}\\[-1pt]{\scriptsize$\pm #2$}\end{tabular}}
\caption{Reaction classification, accuracy and macro-F1 in \%, $\pm$: sample standard deviation over $n$ seeds, three for NPF on ECREACT and CARE. ECREACT, split of \citet{Liu2026Enzyformer}, 4907 test reactions, baselines the mean over three runs. CARE task~2, easy split, 393 test reactions, baselines from \citet{Yang2024CARE}. Best in \textbf{bold}, second \underline{underlined}.}
\label{tab:classification}
\resizebox{\linewidth}{!}{%
\begin{tabular}{@{}lrrrr@{\hspace{2em}}lrrrr@{}}
\toprule
\multicolumn{10}{@{}l@{}}{%
\begin{tabular}{@{}lrrr@{}}
\textit{Schneider50k (50 classes)} & $n$ & accuracy & macro-F1 \\
\midrule
DRFP + MLP          & 3 & \mpm{95.52}{0.04} & \mpm{95.47}{0.04} \\
PGNN, readout       & 3 & \mpm{97.67}{0.07} & \mpm{97.65}{0.07} \\
PGNN, firing vector & 5 & \underline{\mpm{98.64}{0.09}} & \underline{\mpm{98.63}{0.09}} \\
NPF, readout        & 5 & \mpm{97.68}{0.09} & \mpm{97.67}{0.10} \\
NPF, firing vector  & 5 & \bmpm{98.80}{0.04} & \bmpm{98.80}{0.04} \\
\end{tabular}} \\
\midrule[\heavyrulewidth]
\textit{ECREACT, Enzyformer split} & EC4 & EC3 & EC2 & EC1 & \textit{CARE task 2, easy} & EC4 & EC3 & EC2 & EC1 \\
\cmidrule(r){1-5}\cmidrule(l){6-10}
BEC-Pred              & -- & 76.9 & 81.2 & 85.8 & CLIPZyme        & 12.2 & 39.9 & 61.8 & 79.9 \\
DRFP + CLM            & -- & 81.7 & 85.6 & 88.3 & DRFP similarity & 59.3 & 77.1 & 85.2 & 90.6 \\
Synthcoder-DistilBERT & -- & 84.4 & 86.7 & \underline{90.9} & CREEP & 39.4 & 66.4 & 79.9 & 92.9 \\
Enzyformer            & -- & \underline{84.6} & \underline{88.2} & \underline{90.9} & NPF, no map & \sumpm{67.85}{0.39} & \smpm{82.27}{0.29} & \smpm{89.74}{0.39} & \smpm{95.76}{0.59} \\
NPF                   & \sbmpm{75.75}{0.18} & \sbmpm{90.20}{0.15} & \sbmpm{92.45}{0.23} & \sbmpm{95.40}{0.23} & NPF + mapper & \sbmpm{70.23}{1.59} & \sumpm{84.31}{0.64} & \sumpm{92.03}{0.78} & \sbmpm{96.78}{1.03} \\
\cmidrule(l){6-10}
& & & & & CREEP + text & 60.3 & \textbf{89.3} & \textbf{93.9} & \underline{96.7} \\
\bottomrule
\end{tabular}}
\end{table}

On ECREACT \citep{Probst2022biocatalysed}, in the split of \citet{Liu2026Enzyformer}, the state-equation readout predicts the EC number to the third level for $90.20\pm0.15$~\% of the test reactions, $5.6$ points ahead of the best retrained method, Enzyformer ($84.6$~\%), and leads at EC levels 2 and 1 by $4.3$ and $4.5$ points. At the full EC number, which no baseline reports, it reaches $75.75\pm0.18$~\%, where the $224$ test reactions whose EC number never occurs in training cap any method at $95.4$~\%. On CARE, most EC numbers have four or fewer training reactions, here the classifier reaches $70.23\pm1.59$~\% at the
full EC number with the maps of the mapper and $67.85\pm0.39$~\% without maps, against $59.3$~\% for DRFP similarity, the best method that reads only the reaction. CREEP with a text description of the reaction reaches $60.3$~\% at the full EC number and leads at EC levels 3 and 2 ($89.3$ and $93.9$ against $84.31$ and $92.03$~\% for NPF), while the structure of the reaction alone recovers the full EC number better than CREEP.

\paragraph{Forward prediction.} On USPTO-480k \citep{Jin2017}, NPF trained on targets of the net predicts the recorded product first for 87.7~\% of the $40\,000$ test reactions under the rule of Appendix~\ref{sec:implementation_details} (Table~\ref{tab:forward}), matching MAELLE (87.2~\%), which learns the movement of the electrons, and above the graph-edit models MEGAN (86.3~\%) and GTPN (83.2~\%). Sequence models such as Graph2SMILES (90.3~\%) and Molecular Transformer (88.6~\%), as well as the electron-edit model NERF (90.7~\%), remain ahead at top-1, and at top-5 NPF (93.7~\%) is level with with NERF (93.7~\%) and MAELLE (93.9~\%). Training NPF on the minimum firing vectors of the mapper instead of the supplied atom maps raises top-1 from 86.8 to 87.71~\%, as on 10~\% of the data over three seeds (Appendix~\ref{sec:app_forward_prediction}), so the model does not need external atom maps. With 1~\% of the training data, the token game predicts the major product first for $67.44$~\% of the test reactions (Table~\ref{tab:subsets}). Without the enabling rule, accuracy is unchanged (87.6~\%), but 5 out of the $40\,000$ predictions break the valence rule (Table~\ref{tab:forward_validity}), which the enabling rule prevents for every value of the weights. Trained on USPTO-480K, the network thus learns the valence rule almost perfectly by itself, but only the net guarantees it.

\begin{table}[ht]
\centering
\footnotesize
\setlength{\tabcolsep}{3pt}
\setlength{\abovecaptionskip}{2pt}
\setlength{\belowcaptionskip}{2pt}
\renewcommand{\arraystretch}{0.95}
\caption{Top-$k$ accuracy comparison on USPTO-480K. Entries taken from the original papers, as collected by
\citet{Nguyen2026}, are marked with $^\dagger$. NPF is scored as in Appendix~\ref{sec:implementation_details}, one seed. With the enabling rule, every NPF prediction stays within the valence budgets of its atoms (Table~\ref{tab:forward_validity}).}
\label{tab:forward}
\begin{tabular}{@{}llcccc@{}}
\toprule
method & prediction modality & top-1 & top-3 & top-5 \\
\midrule
Molecular Transformer$^\dagger$~\citep{Schwaller2019} & SMILES         & 0.886 & 0.935 & 0.942 \\
Graph2SMILES$^\dagger$~\citep{Tu2022}                 & SMILES         & 0.903 & \textbf{0.940} & \textbf{0.948} \\
MEGAN$^\dagger$~\citep{Sacha2021}                     & Graph edits    & 0.863 & 0.924 & 0.940 \\
GTPN$^\dagger$~\citep{do2019}                         & Graph edits    & 0.832 & 0.860 & 0.865 \\
NERF$^\dagger$~\citep{Bi2021}                         & Electron edits & \textbf{0.907} & 0.933 & 0.937 \\
MAELLE~\citep{Nguyen2026}                             & Electron flow  & 0.872 & 0.930 & 0.939 \\
\midrule
NPF, recorded maps (ours)                             & Token game     & 0.868 & 0.921 & 0.934\\
NPF, targets of the net (ours)                        & Token game     & \underline{0.877} & {0.927} & {0.937}\\
NPF, targets of the net, no enabling rule (ours)      & Token game     & {0.876} & \underline{0.927} & \underline{0.938}\\
\bottomrule
\end{tabular}
\end{table}

\paragraph{Reaction mechanisms.} We use the elementary steps of FlowER~\citep{Joung2025} in their published split (Appendix~\ref{sec:implementation_details}). Given the reactants of a step, a model predicts its products, and a prediction is correct if its canonical SMILES without atom maps and stereochemistry equals that of the recorded products. We report the top-k step accuracy and the top-k pathway accuracy of \citet{Joung2025}, computed as in their evaluation script.

Every top-1 prediction of both nets is a valid molecule (Table~\ref{tab:mechanism_validity}), without a filter, while FlowER discards its invalid samples, about 5~\%, before ranking, and 70 to 79~\% of the top-1 predictions of the sequence models are valid. Every prediction of both nets conserves atoms, protons and electrons, which 17 to 33~\% of the predictions of the sequence models do. The electron net predicts the recorded step first for 90.47~\% of test steps (Table~\ref{tab:mechanism_steps}), ahead of every baseline, the best being FlowER-large (89.74~\%), and within 0.8 points of the 91.27~\% upper limit. The electron net leads at every k with 8.2M parameters, against 16M for FlowER-large. On whole reactions it recovers the pathway at top-1 for $94.49\pm0.36$~\% (Table~\ref{tab:mechanism_pathways}), against 92.51~\% for Graph2SMILES and 92.50~\% for FlowER-large, the best baselines, and 88.97~\% for FlowER.


The arrow net, which moves only pairs, scores $88.85$~\% top-1, and its difference to the electron net comes from the $1.79$~\% of test steps that move single electrons (Appendix~\ref{sec:app_reaction_mechanisms}).

\begin{table}[t]
\centering
\footnotesize
\setlength{\tabcolsep}{3pt}
\caption{Top-$k$ accuracy (\%) on the FlowER test set for individual steps ($162\,002$) and full pathways ($28\,049$ reactions). A pathway counts as correct if some route to a recorded end product has every step in the top $k$. Steps from PMechDB and RMechDB are excluded. The upper limit reflects steps with more than one valid outcome. NPF ranks by beam probability. Baselines rank by sample frequency and come from Fig.~2 of \citet{Joung2025}. Best in \textbf{bold}, second \underline{underlined}.}
\label{tab:mechanism_steps}
\label{tab:mechanism_pathways}
\resizebox{\linewidth}{!}{%
\begin{tabular}{@{}lrrrrrrrrr@{}}
\toprule
& & \multicolumn{4}{c}{step, top-$k$} & \multicolumn{4}{c}{pathway, top-$k$} \\
\cmidrule(lr){3-6}\cmidrule(l){7-10}
method & params & 1 & 2 & 3 & 5 & 1 & 2 & 3 & 5 \\
\midrule
Molecular Transformer & 12M  & 88.75 & 96.62 & 98.42 & 98.92 & 88.31 & 95.64 & 97.02 & 97.59 \\
Graph2SMILES          & 18M  & 89.09 & 96.75 & 98.08 & 98.66 & \underline{92.51} & 95.94 & 97.03 & 97.76 \\
Graph2SMILES+H        & 18M  & 87.39 & 95.31 & 96.87 & 97.67 & 89.22 & 93.70 & 95.10 & 96.19 \\
FlowER                & 7M   & 88.48 & 96.42 & 97.96 & 98.60 & 88.97 & 94.75 & 96.48 & 97.39 \\
FlowER-large          & 16M  & \underline{89.74} & \underline{97.40} & \underline{98.66} & \underline{99.13}
                             & 92.50 & \underline{96.89} & \underline{98.15} & \underline{98.61} \\
\midrule
NPF, arrow net        & 7.3M & $88.85 \pm 0.04$ & $96.13 \pm 0.06$ & $97.33 \pm 0.07$ & $97.71 \pm 0.14$
                             & $90.54 \pm 0.05$ & $94.20 \pm 0.09$ & $94.88 \pm 0.08$ & $95.19 \pm 0.13$ \\
NPF, electron net     & 8.2M & $\mathbf{90.47} \pm 0.09$ & $\mathbf{97.85} \pm 0.08$ & $\mathbf{99.08} \pm 0.08$ & $\mathbf{99.46} \pm 0.14$
                             & $\mathbf{94.49} \pm 0.36$ & $\mathbf{98.46} \pm 0.21$ & $\mathbf{99.21} \pm 0.13$ & $\mathbf{99.52} \pm 0.16$ \\
\midrule
upper limit           &      & 91.27 & 98.89 & 99.79 & 99.99 & -- & -- & -- & -- \\
\bottomrule
\end{tabular}
}
\end{table}

\FloatBarrier
\subsection{What does each Petri component contribute?}
\label{sec:load_bearing}

We evaluated, on Schneider50k for forward prediction and classification, the proposed Petri net architecture against counterparts of equal size but without each component we introduced (Appendix~\ref{sec:app_load_bearing}). Playing the net one bond change at a time, as the stochastic net from Section~\ref{sec:preliminaries} fires, raises the greedy top-1 from $79.72$ to $93.34$~\%. Enabling, introduced in Section~\ref{sec:enforcing}, makes every prediction valence-valid for every value of the weights (Proposition~\ref{prop:validity}) at no cost in accuracy ($p=0.22$), whereas the generic one-shot model breaks a budget in $2.79$~\%, and the token game without the mask in $0.02$~\%, and in $0.16$~\% after a shift to $5\,000$ USPTO-480K test reactions.  The pure state equation readout ($w\equiv1$, Section~\ref{sec:state_equation_readout}) is invariant to an inserted far methyl group. When trained on subsamples of $1\,000$ and 250 reactions, the pure readout is ahead of its generic counterpart, the firing vector (Section~\ref{sec:preliminaries}) only on 250. A D-MPNN on the condensed graph of reaction by ~\citet{Heid2021machine} with the same maps is behind the firing vector except on larger molecules, where the two are level  (Appendix~\ref{sec:app_reaction_classification}).

\section{Conclusion}
In this work, we asked what a neural network looks like when its outputs obey the rules of a Petri net for every value of its weights. The theory answers that it has to keep the firing form and the enabling rule, while it can only learn the rate law or a readout of the firing vector. On the valence net, atom mapping, reaction classification, and forward prediction become three tasks on a firing vector. The minimum firing vector maps atoms without learned weights. It is competitive across a range of reaction data sources, and its tied optima replace learned atom maps as the training targets of the forward model. On USPTO-480K, NPF trained on these targets is level with or ahead of the electron-flow model MAELLE and the graph-edit models, and ahead of the NPF trained on the supplied maps. The state-equation readout predicts EC numbers ahead of published methods on ECREACT and at the full EC number on CARE. On the electron net, the token game predicts the elementary steps of FlowER ahead of published baselines. The enabling mask makes validity a guarantee at no cost in accuracy, while the readout and the firing vector gain on small training sets but no on distribution shifts in molecule size.

\subsection*{Acknowledgements}
We would like to thank Dr. Alma Ademovic Tahirovic for sharing her insights and the useful and enjoyable discussions on neural Petri nets as applied in her own work.
\subsection*{Reproducibility statement}
We have taken care to provide training data, scripts, and logs together with the source code to assist reproducibility. The source code, data, scripts, and logs are available at \url{https://github.com/daenuprobst/npf}. Proofs are in Appendix~\ref{sec:proofs} and Appendix~\ref{sec:implementation_details} describes the data, splits, hyperparameters, and metrics.



\bibliographystyle{plainnat}
\bibliography{references}

\newpage
\appendix
\section{Proofs}
\label{sec:proofs}

Nets, markings, enabling, the incidence matrix $C$, and the attributes $e_p$ and $a_t$ are those introduced in Section~\ref{sec:preliminaries}. For the consumers of $p$, we write $p^{\bullet} = \{ t : \mathrm{Pre}(p,t) > 0 \}. \text{ If } {}^{\bullet}t \cap t^{\bullet} = \emptyset$ for all $t$, the net is pure, meaning no transition takes from and gives to the same place. $\mathbf{1}_p$ is the unit vector of place $p$. On the valence net of the precursor atoms $V$, we write $b_{ij}=m_B$ for the bond order and $s_i=m_{S_i}$ for free valence. Transitions are $t_{ij}^{+\delta}: \delta S_i + \delta S_j \rightarrow \delta B_{ij}$ and their reverse is $t_{ij}^{-\delta}$, for every pair $i<j$ of atoms and every $\delta \in \left\{ \frac{1}{2}, 1, \frac{3}{2}, 2, 3 \right\}$. All columns of one pair are multiples of $\mathbf{1}_{B_{ij}}-\mathbf{1}_{S_i}-\mathbf{1}_{S_j}$, so $C$ has rank $\binom{|V|}{2}$ and $\ker C^\top$ is spanned by the $|V|$ valence budgets $x^{(i)} = \mathbf{1}_{S_i} + \sum_j \mathbf{1}_{B_{ij}}$, with ${x^{(i)}}^\top m = s_i + \sum_j b_{ij}$. Arc weights and markings are all multiples of $\frac{1}{2}$, so doubling them gives a discrete net. An atom map $\pi$ seats the product marking on the precursor places as $P_\pi m_B$ (see Section~\ref{sec:preliminaries}). The implementation masks transitions that are not enabled with a log-rate of $-10^4$ instead of $-\infty$, realising Proposition~\ref{prop:validity} when the log-rate of STOP exceeds $-10^3$. 

The inverse problem uses the two orthogonal decompositions
\begin{equation}
\label{eq:decomp}
    \mathbb{R}^T=\mathrm{im} C^\top\oplus\ker C,\qquad \mathbb{R}^P=\mathrm{im}C\oplus\ker C^\top,
\end{equation}

 and the pseudoinverse $C^+$, where $C^+C$ and $I-C^+C$ project orthogonally onto $\mathrm{im}C^\top$ and $\ker C$. Therefore, every $\sigma$ splits into one part in $\mathrm{im}C^\top$, which is determined by the observed change as $C^+(m-m_0)$, and another part containing cycles in $\ker C$. $\mathrm{im}$, the image, is all that can be produced by the matrix, while $\ker$, the kernel, is all that it maps to zero. In chemistry terms, $\mathrm{im}C$ holds the changes of bonds and free valence that bond changes can cause, $\ker C^\top$ the valence budgets, $\ker C$ the bond changes that cancel each other out, and $\mathrm{im}C^\top$ the part of the bond changes determined by the products.

\subsection{Conservation is equivalent to the firing form}

\begin{proposition}[The conservation is equivalent to the firing form] 
\label{prop:conservation}
Assume the $P$-invariants span $\ker C^\top$. An update $m\mapsto m+u(m)$ satisfies $x^\top(m+u(m))=x^\top m$ for all markings $m$ and all $P$-invariants $x$ iff $u(m)\in\mathrm{im}C$ for all $m$, i.e. $u=C\circ v$ for some $v:\mathbb{R}^P_{\geq0}\rightarrow\mathbb{R}^T$.
\end{proposition}

\begin{proof}
Fix $m$. The condition $x^\top u(m) = 0$ for all $x \in \ker C^\top$ says $u(m) \in (\ker C^\top)^\perp = \operatorname{im} C$. If $u(m) \in \operatorname{im} C$ for all $m$, put $v(m) := C^+ u(m)$, then $C v(m) = C C^+ u(m) = u(m)$. Conversely, $x^\top C v = 0$ for every $P$-invariant $x$.
\end{proof}

Note that $v$ can have negative entries for reverse firings.

\paragraph{Conservation fails under generic parameters.} A model updates a marking as $m^+_p=m_p+r_\theta(h_p)+c$, with a free output bias $c\in\mathbb{R}$. This is the same that message-passing models with an MLP readout do. Let $x$ be a $P$-invariant where $\sum_p x_p\neq0$. Two parameter vectors differing only in $c$ give values of $x^\top(m^+-m)$ that differ by $(c-c^\prime)\sum_p x_p\neq0$. This means that for every input and every value of the other weights, at most one value of $c$ conserves $x$. Thus, the conserving parameters form a set of Lebesgue measure zero, giving gradient training no reason to find. In short, a loss can drive such a model towards conservation, but nothing in the architecture keeps it there on new inputs.

\subsection{The enabling rule is forced}
\begin{proposition}[The enabling rule is forced.]

\label{prop:enabling_rule}
    $v_t(m)=f\left((m_p,e_p,\mathrm{Pre}(p,t))_{p\in^\bullet t},a_t\right)$ is a rate law is local shared across nets, where $f\geq0$ continuous. If on every pure net and from every $m(0)\in\mathbb{R}^P_{\geq0}$ the equation $\Dot{m}=Cv(m)$ has a solution that stays in $\mathbb{R}^P_{\geq0}$ on some $[0,\epsilon)$, then $f=0$ whenever an input place is empty.
\end{proposition}

\begin{proof}
    Suppose $f>0$ at an argument with $m_q=0$ for some $q\in^\bullet t$. Then consider a pure net consisting of $t$, input places with these attributes and weights, and one new output place, letting $m(0)$ have these values. By hypothesis, there is a solution with $m(s)\geq0$ on some $[0,\epsilon)$. The place $q$ has no producer, so $0\leq m_q(\tau)=-\mathrm{Pre}(q,t)\int_0^\tau f\mathrm{d}s$ along the solution. As $f\geq0$, the integrand vanishes for most $s$, and, by continuity of $f$ and of the solution for every $s$, especially at $s=0$, there is a contradiction.
\end{proof}

The rate law is the same formula everywhere, which means that if it should fire a transition with an empty input place, we could build a one-transition net that creates this situation. As nothing refills the empty place, it would go negative. However, as negative markings are forbidden, the rate has to be zero when an input place is empty. This is important because this tells us that the enabling rule is not a design choice that we impose, but a consequence of non-negativity. Any learned or handwritten rate law that is local, shared, continuous, and that keeps markings nonnegative has to respect it, while one that violates it will, at some point, produce impossible states. For one firing at a time, as in the jump chain of Section~\ref{sec:preliminaries} and in the token game, the rule is force without locality. On a pure net, a transition $t$ that is not enabled at $m$ has $m_p < \mathrm{Pre}(p,t)$ for some input place $p$, so a rate law that keeps every marking non-negative has to set $\lambda_t(m)=0$. The valence net is pure, and this zero rate is the mask of Section~\ref{sec:enforcing}, applied to $m^\circ$ of Proposition~\ref{prop:validity}.

\subsection{Cancellation and invariance of the readout}
\paragraph{Setting.} The two graphs $G_A=(V_A,x^A,b^A)$ and $G_
B=(V_B,x^B,b^B)$ with node attributes $x$ and bond types $b(\cdot,\cdot)$, where 0 = no bond, and $\pi=V_B\rightarrow V_A$ is injective. The $R$-ball of node (atom) $i$ is the attributed (sub) graph induced on the nodes within distance $R$ of $i$. $\psi_G(i)$ only depends on the isomorphism class of the $R$-ball of $i$ in $G$. Let $D\subseteq V_B$ contain $i$ if $x^B_i\neq x^A_{\pi(i)}$, or $b^B(i,k)\neq b^A(\pi(i),\pi(k))$ for some $k\in V_B$, or $b^A(\pi(i),u)\neq0$ for some $u\notin\pi(V_B)$, and let $D_R=\{i:\mathrm{dist}_B(i,D)\leq R\}$.

In chemistry terms, we have two molecules $A$ and $B$ that are atom-to-atom mapped by $\pi$. The featuriser $\psi_G(i)$ describes the atom $i$ by its chemical environment within $R$ bonds. $D$ is the set of atoms in $B$ where elements or bonding differs from their counterparts in $A$ and $D_R$ adds every atom within a radius of $R$. This is conceptually similar to the DRFP fingerprint.

\begin{proposition}[Cancellation and invariance]
\label{prop:cancellation}
    Let $r=\sum_{i\in V_B}\psi_B(i)-\sum_{j\in V_A}\psi_A(j)$. Then $r=\sum_{i\in D_R}[\psi_B(i)-\psi_A(\pi(i))]-\sum_{j\notin\pi(V_B)}\psi_A(j)$. Specifically, attaching the same attributed subgraph by one edge to $i$ and $\pi(i)$, where $i$ is more than $2R$ away from $D$ and the attributes of $i$ and $\pi(i)$ change identically, leaves $r$ unchanged.
\end{proposition}

\begin{lemma}
\label{lem:ball}
    If $\mathrm{dist}_B(i,D)>R$, then $\pi$ restricts to an isomorphism of the $R$-ball of $i$ in $G_B$ onto the $R$-ball of $\pi(i)$ in $G_A$. Specifically, it preserves distances to the root, and $\psi_B(i)=\psi_A(\pi(i))$.
\end{lemma}

\begin{proof}
    No node $u$ of $R$-ball $U$ of $i$ lies in $D$, this means that the neighbours of $\pi(u)$ in $G_A$ are exactly the images of the neighbours of $u$, with equal bond types and none of them is unmatched. By induction on $d\leq R$, $\pi$ maps the nodes at a distance $d$ from $i$ onto the nodes at distance $d$ from $\pi(i)$. Attributes agree on $U$ and bond types agree on $U\times U$ due to $U\cap D=\emptyset$.
\end{proof}

\begin{proof}[Proof of Proposition~\ref{prop:cancellation}]
As $\pi$ is injective, $V_A$ is the disjoint union of $\pi(V_B)$ and the unmatched nodes, so
$r=\sum_{i\in V_B}[\psi_B(i)-\psi_A(\pi(i))]-\sum_{j\notin\pi(V_B)}\psi_A(j)$,
and the terms with $i\notin D_R$ vanish by Lemma~\ref{lem:ball}. For invariance, attach $H$ to $i_0$ and to $\pi(i_0)$, where $\mathrm{dist}_B(i_0,D)>2R$, and extend $\pi$ by the identity on $H$. As $H$ is pendant and fully matched, and $i_0$ and $\pi(i_0)$ change identically, the sets $D$, $D_R$ and the unmatched nodes are unchanged, so it suffices to check that no term of the formula changes. For $i\in D_R$ the triangle inequality gives $\mathrm{dist}_B(i,i_0)>R$, so $i_0$ and $H$ lie outside the $R$-ball of $i$ in $G_B$. By Lemma~\ref{lem:ball} applied to $i_0$, the $R$-ball of $\pi(i_0)$ in $G_A$ is the image of the $R$-ball of $i_0$, hence contains no $\pi(i)$ with $i\in D_R$ and no unmatched node. Therefore, every term of the formula is unchanged.
\end{proof}

For chemistry, this means the difference fingerprint $r$ only depends on the atoms at the reaction centre and within $R$ bonds of this centre, and any leaving atoms. All other atoms have the same environment on both sides and cancel each other out. Invariance is a consequence of this, as decorating reactants and products identically far from the reaction centre leaves $r$ unchanged. In short, atoms are split into three groups, those that changed or are close to a change, those far away from any change and those that leave. According to Lemma~\ref{lem:ball}, the middle group contributes equally to sums and disappears, leaving only the first group plus the leaving atoms.

Similar to message passing, the encoder updates the embedding of each atom from its bonded neighbours for $K$ rounds. Atom terms have $R=K$, and bond terms that read the embeddings of an atom and of its neighbours along bonds of one type have $R=K+1$. In Section~\ref{sec:load_bearing}, we use $K=3$ and attachment sites at least 9 bonds from $D$. A readout $\rho\left(\sum_B\psi,\sum_A\psi\right)$ with a generic $\rho$ has neither cancellation nor invariance, as the sums themselves change with every spectator atom.

\subsection{Validity of the token game}
\begin{proposition}[Validity] 
\label{prop:validity}
Let $m^\circ_A=m_A+\frac{1}{2}\sum_i \mathbf{1}_{S_i}$ be the marking of the precursors with half a token of tolerance on every slack place. For all weights, given $m$ is a marking visited by the token game described in Section~\ref{sec:chemical_reactions}, then $m^\circ=m+\frac{1}{2}\sum_i \mathbf{1}_{S_i}$ is reachable from $m^\circ_A$ in the valence net. Specifically $m$ is non-negative on the bond places and at least $-\frac{1}{2}$ on the slack places, has the valence budgets of $m_A$, and $\sum_j b_{ij}\leq x^{(i)^\top}m_A+\frac{1}{2}$ for every atom $i$.

\end{proposition}

\begin{proof}
    By induction over the steps. A step re-types one bond place $B_{ij}$ from its order to another order in $\{0,1,\frac{3}{2},2,3\}$. If the order grows by $\delta$, the step is a firing of $t^{+\delta}_{ij}$. The game admits it only if $\delta\leq s_i+\frac{1}{2}$ and $\delta\leq s_j+\frac{1}{2}$, which is the enabling rule at $m^\circ$, whose slack places hold $s_i+\frac{1}{2}$. If the order is reduced by $\delta$, the step is a firing of $t^{-\delta}_{ij}$, which is enabled. Therefore, the next marking is reachable, non-negative, and the valences, being $P$-invariant, remain unchanged.
\end{proof}

A few notes. (1) The half token is needed for aromatic bonds. (2) We make the following design choices given the data: A game re-types each bond place once per reaction at most and fires twelve times at most. (3) The budget of an atom depends on its element and on its state in $m_A$ as an atom already cationic or hypervalent there keeps its spare tokens. (4) Training states $m_A+C\sigma_F$ apply random parts of a target firing vector without order and does not have to be reachable. The guarantee only concerns the markings visited at inference. (5) A new bond needs $\delta\geq1$, so a carbon with a valence budget of four, with $\sum_j b_{ij}\leq4\frac{1}{2}$, never gets a fifth bond. Aromatic carbons at a ring fusion hold $4\frac{1}{2}$ in $m_A$ and aromatic carbons with an exocyclic double bond hold $5$. Only the Kekulé net of Appendix\ref{sec:implementation_details} caps their total bond order at four.

\begin{lemma}[Realisability on the valence net]
\label{lem:realisability}
    Let $\sigma$ be a firing vector of the valence net with $\sigma_t\leq1$ for every transition $t$ and at most one fired transition on every bond place. Let $m=m_A+C\sigma$ be non-negative. Then every sequence that fires the transitions of $\sigma$ once each, with the bond-breaking ones first, is enabled at $m_A$, so $\sigma$ is the count vector of a firing sequence enabled at $m_A$. The same holds with $m^\circ_A$, the precursors with half a token more on every slack place (see Proposition~\ref{prop:validity}), in place of $m_A$ whenever $m^\circ=m^\circ_A+C\sigma$ is non-negative.
\end{lemma}

\begin{proof}
    A firing $t^{-\delta}_{ij}$ of $\sigma$ is enabled at every marking before it, as no other firing of $\sigma$ changes $B_{ij}$ and $\delta\leq b_{ij}$ at $m_A$ as $m\geq0$. Breaking a bond only adds tokens to slack places. Let $s^\prime$ be the slack after all breaks, and let $t^{+\delta}_{ij}$ be a bond-forming firing after some of the others. Every forming firing removes tokens, which means that $S_i$ then holds at least $s^\prime_i$ minus the tokens of all forming firings at $i$ except this one, that is, at least $m_{S_i}+\delta\geq\delta$. Same for $S_j$. Therefore, every step is enabled and the same holds for $m^\circ_A$.
\end{proof}

\begin{figure}[h]
    \centering
    \setchemfig{atom sep=1.1em, bond offset=0.9pt, double bond sep=1.8pt, bond style={line width=0.6pt}}%
\begin{tikzpicture}[
    font=\footnotesize, >={Stealth[length=1.5mm, width=1.2mm]},
    mol/.style={inner sep=0pt, font=\scriptsize},
    flow/.style={->, line width=0.6pt, black!75},
    excluded/.style={->, line width=0.6pt, black!30, densely dashed},
    marking/.style={circle, draw, line width=0.6pt, fill=black!8, minimum size=2.6mm, inner sep=0pt},
    unreached/.style={circle, draw=black!30, densely dashed, line width=0.6pt, fill=white, minimum size=2.6mm,
                      inner sep=0pt},
    prob/.style={font=\scriptsize, text=magenta!75!black, fill=white, inner sep=0.8pt},
    head/.style={font=\footnotesize\bfseries, anchor=west, inner sep=0pt},
    note/.style={font=\scriptsize, text=black!70, inner sep=0pt, align=left},
    set/.style={font=\scriptsize, inner sep=1pt}]

  \node[head] at (0, 4.72) {(a) the most probable order};

  \node[mol] at (1.2, 4.2) {\chemfig{OHC-*6(-=-(-[,,,,draw=magenta!75!black]B|{(OH)_2})=-=)}};
  \node[mol] at (1.2, 3.42) {\chemfig{Br-*5(=-O-=-)}};

  \node[mol] at (4.75, 4.2) {\chemfig{OHC-*6(-=-(-[:0,0.55,,,draw=none]{\odot})=-=)}};
  \node[mol] at (4.75, 3.42)
       {\chemfig{Br-[,,,,draw=magenta!75!black]*5(=-O-=-)}\quad\chemfig{{\odot}-[:0,1,,,draw=none]B|{(OH)_2}}};

  \node[mol] at (8.55, 4.2)
       {\chemfig{OHC-*6(-=-(-[:0,0.55,,,draw=none]{\odot}-[:0,1.9,,,{draw=magenta!75!black,densely dashed}]{\odot}-[:0,0.55,,,draw=none]*5(=-O-=-))=-=)}};
  \node[mol] at (8.55, 3.42)
       {\chemfig{Br-[:0,1.35,,,draw=none]{\odot}}\qquad\chemfig{{\odot}-[:0,1,,,draw=none]B|{(OH)_2}}};

  \node[mol] at (12.15, 4.2) {\chemfig{OHC-*6(-=-(-*5(=-O-=-))=-=)}};
  \node[mol, text=black!55] at (12.15, 3.42) {HBr\qquad HB(OH)$_2$};

  \foreach \x/\m in {1.2/{m_0=m_A}, 4.75/m_1, 8.55/m_2, 12.15/{m_3=m_B}}
    \node[font=\scriptsize] at (\x, 2.97) {$\m$};
  \foreach \x/\verb/\bond/\p in {2.5/break/C--B/0.68, 6.25/break/C--Br/1.00, 10.15/form/C--C/1.00} {
    \draw[flow] (\x, 3.81) -- ++(0.8, 0);
    \node[note, text=black, anchor=south, align=center] at (\x + 0.4, 3.88) {\verb\\\bond};
    \node[prob, anchor=north] at (\x + 0.4, 3.74) {\p};
  }

  \node[head] at (0, 2.42) {(b) all orders};
  \node[note, anchor=north west] at (0, 2.14) {$a$\,\ break C--B\\ $b$\,\ break C--Br\\ $c$\,\ form C--C};

  \node[marking, label={[set]below:$\emptyset$}] (n0) at (1.2, 1.04) {};
  \node[marking, label={[set]above:$b$}] (na) at (4.75, 1.66) {};
  \node[marking, label={[set]below:$a$}] (nb) at (4.75, 1.04) {};
  \node[unreached, label={[set, text=black!45]below:$c$}] (nc) at (4.75, 0.42) {};
  \node[unreached, label={[set, text=black!45]above:$bc$}] (nac) at (8.55, 1.66) {};
  \node[marking, label={[set]below:$ab$}] (nab) at (8.55, 1.04) {};
  \node[unreached, label={[set, text=black!45]below:$ac$}] (nbc) at (8.55, 0.42) {};
  \node[marking, label={[set]below:$abc$}] (nabc) at (12.15, 1.04) {};

  \draw[excluded] (n0) -- (nc);
  \draw[excluded] (na) -- (nac);
  \draw[excluded] (nb) -- (nbc);
  \draw[excluded] (nc) -- (nac);
  \draw[excluded] (nc) -- (nbc);
  \draw[excluded] (nac) -- (nabc);
  \draw[excluded] (nbc) -- (nabc);
  \draw[flow] (n0) -- node[prob, pos=0.5] {0.32} (na);
  \draw[flow] (n0) -- node[prob, pos=0.5] {0.68} (nb);
  \draw[flow] (na) -- node[prob, pos=0.5] {1.00} (nab);
  \draw[flow] (nb) -- node[prob, pos=0.5] {1.00} (nab);
  \draw[flow] (nab) -- node[prob, pos=0.5] {1.00} (nabc);

  \draw[flow] (nabc) -- node[prob, above=1pt, fill=none] {1.00} ++(0.7, 0) node[right=1pt, font=\scriptsize] {\textsc{stop}};
  \node[note, anchor=north east, align=right] at (13.5, 0.66)
       {$P(\text{product}) = 0.9999$\\ 2 of 6 orders enabled};

  \pgfresetboundingbox
  \useasboundingbox (-0.03, 0.1) rectangle ++(13.6, 4.8);
\end{tikzpicture}%
    \caption{A bromo Suzuki-type coupling as a token game of the model trained without precomputed mappings. (a) Showing the most probable order, each bond that changes in magenta, dashed if formed. $\odot$ is a token that a broken bond returns to the slack place of its atom. The order is valence bookkeeping, not the Pd mechanism, which would break C-Br first. (b) Forming C-C requires a token of free valence on both aryl carbons that can only be supplied by the two breaks, meaning enabling excludes the four orders that form C-C before the two bond breaks. The two remaining orders fire the same transitions and reach the same product.}
    \label{fig:tokengame}
\end{figure}
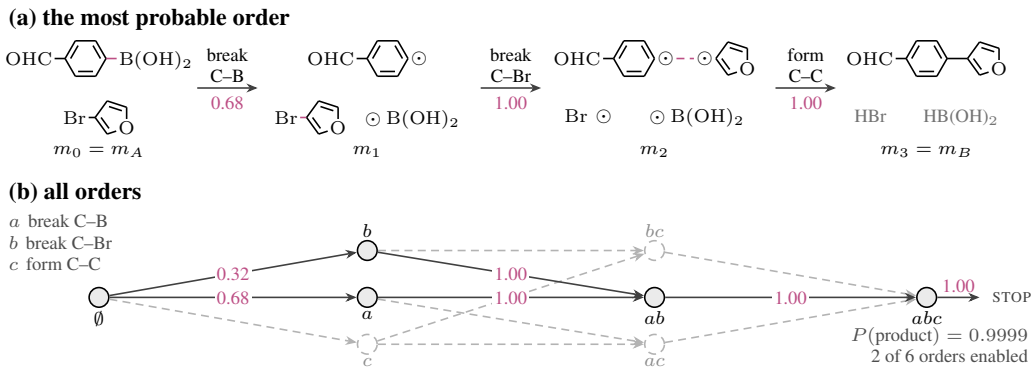

This means that the converse of the state equation, which fails in general as discussed in Section~\ref{sec:preliminaries} holds for the firing vectors that the mapper returns and that the token game trains on, whenever the seated product marking is non-negative. In chemistry terms this means that breaking every bond first frees all the valence that the formations need. Therefore, the valence rule can forbid an order of the bond changes, such as a new bond to a carbon that still has four bonds, however, it can never forbid a product in which every atom stays within its valence budget, as long as each bond changes at most once because breaking the bonds first is then always an allowed order. So if we keep only training targets with an allowed order (Appendix~\ref{sec:implementation_details}), it only ever asks that their products have valid valences.

\subsection{What two markings can and cannot identify}
\label{sec:markings_identify}
Our layers are shaped by three gauge freedoms that limit what any estimator recovers from pairs of markings.

\begin{itemize}
    \item \textbf{Rate scale.} The embedded jump chain only uses $\lambda$ through $\lambda_t(m) / (\lambda_{\mathrm{STOP}}(m) + \sum_{t'} \lambda_{t'}(m))$, this means that a common factor $e^{c(m)}$ on all the rates at a marking does not change any observation. Untimed pairs, as we have in chemistry, only identify the relative law. However, given timed chemistry data means that reactants and products by themselves could not tell fast kinetics observed briefly from slow kinetics observed for long. A $P$-invariant term $x^\top m$ in every log rate, which in chemistry is the valence budget of an atom, is such a factor that is constant along every path. Therefore, which bond changes happen would solely depend on how the possible changes compare and never on an overall speed or concentration.
    \item \textbf{The kernel of C.} Any given pair of markings only fixes $\sigma$ modulo $\ker C$, which is large in the valence net. This means that a bond formed and then broken again, or a double bond made in one step instead of two, does not leave any trace in the products. Excluding both those cases, the token game re-types each bond place only once, making $\sigma\mapsto C\sigma$ injective and $\sigma$ a function of the change $P_\pi m_B-m_A$ under a correspondence $\pi$. The map $\pi\mapsto\sigma(\pi)$ is constant on correspondences that differ by a symmetry of an unchanged part, moving the ambiguity to the choice of $\pi$, and the mapper whose minimum contains no cycling, is the right encoder for the inverse problem. For chemistry, this means that swapping two equivalent atoms, while changing the atom map, does not affect the bond changes. Therefore, uncertainty is limited to which atom is which, and this is left to the atom mapper.
    \item \textbf{Orders inside a trace.} Assuming that transitions $a$ and $b$ are enabled at $m$ in any order and nothing enabled after both, the chain fires $a$ or $b$ first with a probability of $x$ or $y$, respectively, and stops with a probability $z=1-x-y$, and stops after the first firing with a probability of $s_1$ or $s_2$. The four final markings are seen with the probabilities $z,xs_1,ys_2$ and $x(1-s_1)+y(1-s_2)$, so the order is never recovered. Therefore, two markings identify which transitions fire, not the order in which they do. This also applies to chemistry, and not required by our models.
\end{itemize}

In summary, products fix the net bond changes and not the route, the speed, or anything constant along the way, and the layers are shaped in a way that they learn exactly the part that is fixed and nothing else.

\section{Implementation details}
\label{sec:implementation_details}

\textbf{Training.} All our networks are trained with AdamW with a weight decay of $10^{-5}$ as optimiser and a one-cycle learning rate schedule peaking at $10^{-3}$, and at $4\times10^{-4}$ for the rate law with width 256 after a warm-up over the first 10~\% of the steps with gradients clipped at norm 1. For the token game on USPTO-MIT we use $4\cdot10^{-4}$. The table gives the hyperparameters of the models we report on Schneider50k and USPTO-MIT. The DRFP baseline is a perceptron with one hidden layer of $1\,024$ units and a dropout of 0.2, trained with AdamW and a weight decay of $10^{-4}$ for 40 epochs at a constant learning rate of $10^{-3}$.

\begin{center}\small\setlength{\tabcolsep}{5pt}
\resizebox{\linewidth}{!}{%
\begin{tabular}{lrrrrr}
\toprule
 & token game & width 256 & one-shot counterparts & classifier & classifier, firing vector\\
\midrule
width & 128 & 256 & 128 & 128 & 128\\
message-passing rounds & 6 & 8 & 6 & 3 & 3\\
attention layers & 4 & 8 & 4 & -- & --\\
epochs & 60 & 60 & 60 & 60 & 60\\
batch size & 32 & 32 & 32 & 64 & 64\\
parameters & 1.6M & 9.8M & 1.6M & 2.1M & 2.5M \\
\bottomrule
\end{tabular}
}
\end{center}

\begin{itemize}
    \item \textbf{Features.} We one-hot encode the element, degree, formal charge, number of hydrogens, aromaticity, and ring membership. Atom-map numbers are removed before canonical atom ordering so they are not leaked by either features or order. The encoder has one message function per bond type, residual updates, and layer normalisation.
    \item \textbf{Token game.} In addition to the fixed features above, the input contains two quantities that change with the marking, namely the hydrogen count of every atom at the current marking (which is the slack place, its free valence minus capacity and negative charge) and whether a bond place of the atom has fired. Pair logits come from $[h_i+h_j,h_i\odot h_j]$ (the pair feature read by the rate law for the bond place $B_{ij}$), the current bond type, a same-molecule flag, and a fired flag. Batch size is 32. Training fires a random part $\sigma_F \le \sigma^*$ of some $\sigma^* \in \Sigma^*$, and the loss is $-\log \sum_t \operatorname{Pr}(t \mid m_A + C\sigma_F)$ over the enabled transitions that continue a vector of $\Sigma^*$ containing $\sigma_F$, together with $\mathrm{STOP}$ if $\sigma_F \in \Sigma^*$.
    \item \textbf{Classifier.} We want locality, so there is no attention. Dropout is $0.2$. The participation gate learns whether a precursor molecule holds a seated atom from the maps via a binary cross-entropy of $0.5$. We matched the number of parameters with the generic comparison. The generic head pools each side and adds one learned term per seated atom. When marked, the transition types that fired are predicted by an auxiliary head, with targets from the maps of the mapper of all training reactions, including those with withheld class labels.

    \item \textbf{Forward.} On the USPTO-480K forward prediction, the rate law has width 256, eight message-passing rounds and eight attention layers over the whole marking, 9.8M parameters, and is trained for 60 epochs with batches of 32 reactions, a peak learning rate of $4\cdot10^{-4}$ and mixed precision. The targets are the minimum firing vectors of the mapper that decode to the recorded product, up to eight per reaction. A training reaction without such a vector is left out, which keeps $393\,621$ reactions, against $399\,184$ with the recorded atom maps. After every epoch the model is scored greedily on $1\,500$ validation reactions, and the best epoch is kept. At test time, a beam of width 5 merges count-equivalent sequences. The variant without the enabling rule is the same network with the mask removed, so it may fire any transition whose input places hold tokens, whether or not the atoms have room for the new bond.

\begin{center}\small\setlength{\tabcolsep}{5pt}
\newcommand{\hd}[1]{\begin{tabular}[b]{@{}c@{}}#1\end{tabular}}
\begin{tabular}{@{}lccc@{}}
\toprule
& \hd{token game,\\recorded maps} & \hd{token game,\\targets of the net} & \hd{token game, targets\\of the net, no enabling} \\
\midrule
width                        & 256 & 256 & 256 \\
message-passing rounds       & 8   & 8   & 8   \\
attention layers             & 8   & 8   & 8   \\
epochs                       & 60  & 60  & 60  \\
batch (reactions)            & 32  & 32  & 32  \\
peak learning rate           & $4\cdot10^{-4}$ & $4\cdot10^{-4}$ & $4\cdot10^{-4}$ \\
beam width at test time      & 5   & 5   & 5   \\
training reactions           & $399\,184$ & $393\,621$ & $393\,621$ \\
parameters                   & 9.8M & 9.8M & 9.8M \\
\bottomrule
\end{tabular}
\end{center}
\end{itemize}

\begin{itemize}
\item \textbf{Enzyme classifiers.} The CARE and ECREACT classifiers are the state-equation readout of
Section~\ref{sec:state_equation_readout} with the participation gate, trained with AdamW, a weight decay of $10^{-4}$ and a one-cycle schedule for 40 epochs; the model after the last epoch is evaluated. They have 4.6M and 4.8M parameters, the difference coming from the output layer over $4\,874$ and $5\,192$ EC numbers. The EnzymeMap classifier reads the firing vector of the mapper, as on Schneider50k, and is trained for 60 epochs.

\begin{center}\small\setlength{\tabcolsep}{5pt}
\newcommand{\hd}[1]{\begin{tabular}[b]{@{}c@{}}#1\end{tabular}}
\begin{tabular}{@{}lcccc@{}}
\toprule
& \hd{arrow\\net} & \hd{electron\\net} & \hd{classifier,\\CARE, ECREACT} & \hd{classifier,\\EnzymeMap} \\
\midrule
width                        & 256 & 256 & 128 & 128 \\
message-passing rounds       & 6   & 6   & 3   & 3   \\
attention layers             & 6   & 6   & --  & --  \\
epochs                       & 12  & 12  & 40  & 60  \\
batch (atom pairs / reactions) & $\le 2\cdot10^6$ & $\le 2\cdot10^6$ & 16 & 16 \\
peak learning rate           & $5\cdot10^{-4}$ & $5\cdot10^{-4}$ & $10^{-3}$ & $10^{-3}$ \\
parameters                   & 7.3M & 8.2M & 4.6M\,/\,4.8M & 2.5M \\
\bottomrule
\end{tabular}
\end{center}
\end{itemize}

\begin{itemize}
\item \textbf{Mechanism nets.} The shell capacity of an element is the largest shell of that element seen at least ten times in the training steps. Eight electrons for C, N, O, F, Cl, and Br, twelve for P, S, and I, never below what the atom starts with. Every formal charge stays within a window per element, widened by one while the step progresses. Both nets are trained with AdamW, a weight decay of $10^{-2}$, a one-cycle schedule with a warm-up over the first $5$~\% of the updates and gradients clipped at norm 1. A step of $n$ atoms has $n^2$ bond places, so batches are formed by size rather than by count: steps are sorted by size, and each batch holds steps of similar size up to $2\cdot10^6$ padded atom pairs and at most 256 steps, about 138 steps on average and $10\,309$ batches per epoch, in a random order per epoch. Twelve epochs are about $1.2\cdot10^5$ updates. The pair head has width 64, a step fires at most 16 arrows, and validation top-1 is measured greedily on $3\,000$ validation steps after every epoch. At test time, seed 0 uses a beam of 10 and seeds 1 and 2 a beam of 5.

\begin{center}\small\setlength{\tabcolsep}{5pt}
\newcommand{\hd}[1]{\begin{tabular}[b]{@{}c@{}}#1\end{tabular}}
\begin{tabular}{@{}lcc@{}}
\toprule
& \hd{arrow\\net} & \hd{electron\\net} \\
\midrule
token                              & electron pair & electron \\
width                              & 256 & 256 \\
message-passing rounds             & 6   & 6   \\
attention layers                   & 6   & 6   \\
pair head width                    & 64  & 64  \\
epochs                             & 12  & 12  \\
batch (padded atom pairs)          & $\le 2\cdot10^6$ & $\le 2\cdot10^6$ \\
peak learning rate                 & $5\cdot10^{-4}$ & $5\cdot10^{-4}$ \\
weight decay                       & $10^{-2}$ & $10^{-2}$ \\
arrows per step, at most           & 16  & 16  \\
beam width at test time            & 10 / 5 & 10 / 5 \\
training steps                     & $1\,445\,189$ & $1\,445\,189$ \\
parameters                         & 7.3M & 8.2M \\
\bottomrule
\end{tabular}
\end{center}
\end{itemize}

\paragraph{Kekul\'e structures in the electron net.} Molecules enter both mechanism nets as Kekul\'e structures, in which aromatic rings are written with alternating single and double bonds, as an aromatic bond of order $\frac{2}{3}$ would hold three electrons. Arrow pushing and the bond--electron matrix used by \citet{Joung2025} start from a Kekul\'e structure for the same reason. The valence net, in contrast, keeps aromatic bonds of order $\frac{3}{2}$, as its tokens count bond order rather than electrons, so a half token per atom is enough.

\paragraph{A Kekul\'e variant of the valence net.} Every molecule is kekulised with RDKit~\citep{Landrum2026}, so every bond place holds a whole number of tokens. Three changes are applied. First, a ring usually has more than one Kekul\'e structure, and RDKit may write a product ring in a different one than the precursor ring. This would make every bond of an unchanged ring look changed. Before the firing vector is computed, we therefore redraw each aromatic ring of the product with its double bonds in the same positions as in the precursor, using the atom map to match the atoms, so that a ring the reaction does not touch shows no bond changes. The positions are found by a maximum-weight perfect matching on the atoms that need a double bond. When no perfect matching exists, RDKit's structure is kept. Second, the enabling rule is exact, without the half token of tolerance of Proposition~\ref{prop:validity}. Third, a transition that forms a bond is also enabled only if the atom's total bond order to heavy atoms stays within a maximum per element, measured on the training data. These are, 4 for C and N, 3 for O and 6 for S, so a spare token of capacity cannot give an atom more bonds than it can hold. With whole bond orders, the decoding rule for aromatic rings (Appendix~\ref{sec:implementation_details}, rule~3) is no longer needed. Atom maps are computed on the aromatic molecules and transferred to the Kekul\'e atom order through canonical ranks, so both of the representations use the same maps.

\paragraph{Benchmarks.} \paragraph{Benchmarks.} Atom mapping is evaluated on the Golden set by ~\citet{Lin2021}, whose maps were curated by chemists, on the data sets recently introduced by \citet{Phan2026}, and, for enzymatic reactions, on the $41\,510$ reactions of EnzymeMap from \citet{Heid2023EnzymeMap} against their validated maps. Reaction classification uses the published split of Schneider50k, whose 50 reaction classes were assigned by the rule-based NameRxn~\citep{Schneider2015}. EC classification, predicting the enzyme commission number of an enzymatic reaction, uses three benchmarks, namely, ECREACT~\citep{Probst2022biocatalysed} in the split of \citet{Liu2026Enzyformer}, on which six methods were retrained, task~2 of CARE~\citep{Yang2024CARE} in its easy split, and EnzymeMap at EC level~3, where we compare maps from different sources as input to the same classifier (Table~\ref{tab:enzymemap_classification}). Forward prediction is evaluated on USPTO-MIT (USPTO-480K), the patent reactions of \citet{Jin2017} in their standard split with $40\,000$ test reactions, and the ablations use Schneider50k, including training sets of $1\,995$ and $7\,973$ reactions for forward prediction and of $1\,000$ and 250 for classification. Reaction mechanisms are evaluated on the elementary steps of FlowER~\citep{Joung2025}, where the electron net predicts the products of each step and whole reactions are scored as pathways of steps. Products are compared after aromaticity is perceived again, so the Kekul\'e structure a ring is written in does not matter. A test reaction counts for the pathway accuracy if some route from its reactants to a recorded product has every step within the top $k$.

\paragraph{Capacities of the valence net.} Based on the data used in our experiments, one spare token lets N, O, Cl, Br, and I share a lone pair in one extra bond, where the atom still has eight valence electrons, and carbon has none, the enabling rule thereby keeps neutral and anionic C, N, O, and halogen atoms within the octet. S and P get two spare tokens, so they can form the S=O and P=O bonds of sulfoxides and phosphine oxides, but a sulfide cannot become a sulfone in one reaction. An atom that is already cationic or hypervalent in the precursors keeps its spare token, giving it a looser budget, while a nitro nitrogen could take a fifth bond.

\paragraph{Data.} USPTO-MIT by \citet{Jin2017} holds patent reactions from \citet{Lowe2012}, split into $409\,035$ training, $30\,000$ validation, and $40\,000$ test reactions with reagents among the precursors and without stereochemistry. RDKit parses $409\,006$ training and $39\,994$ test reactions~\citep{Landrum2026}. For training and model selection we keep the $401\,102$ reactions with at most 100 and 70 heavy atoms in the precursors and products, respectively, to speed up training. We also remove reactions with a major product in the Golden set ($940$ for targets of the net, $931$ for precomputed maps). The test set is never filtered, the results are shares of all $40\,000$ reactions, and the 6 unparsable reactions count as failures. The nested 1 \% and 10 \% subsets are the first $4\,090$ and $40\,900$ reactions of one fixed shuffle.

Schneider50k by \citet{Schneider2015} contains $50\,000$ reactions from patents in 50 classes from a rule-based classifier. $49\,994$ are parsable by RDKit. For classification, we use the published split, 200 training reactions per class, for a total of $10\,000$, and 800 per class for testing, for a total of $40\,000$. We use 500 training reactions to choose the best epoch.

The elementary steps of FlowER~\citep{Joung2025} come from USPTO-Full reactions whose mechanisms were imputed with $1\,220$ expert templates for 252 reaction classes, keeping the routes that reach the recorded product, together with the curated polar and radical steps of PMechDB and RMechDB. We use the published split into $1\,445\,189$ training, $15\,744$ validation and $162\,002$ test steps, from $250\,782$, $2\,801$ and $28\,049$ overall reactions. Every atom, hydrogens included, is part of the net, and $2\,905$ test steps ($1.79$~\%) move single electrons.

ECREACT~\citep{Probst2022biocatalysed} is used in the split released with Enzyformer~\citep{Liu2026Enzyformer}: $49\,202$ enzymatic reactions with cofactors and common by-products removed from both sides, randomly split into $39\,404$ training, $4\,891$ validation and $4\,907$ test reactions. Compared to the split by \citet{Probst2022biocatalysed}, which keeps the products of training and test reactions apart, this split is random, and as a reaction can have more than one EC number, a test reaction can occur in training under another EC number. As in Enzyformer, training and validation reactions are pooled for training, two that RDKit cannot parse are dropped, leaving $44\,293$. The labels are the $5\,192$ four-level EC numbers of all splits, and the $224$ test reactions whose EC number does not occur in training count as errors.

CARE task~2~\citep{Yang2024CARE}, easy split, holds out reactions rather than EC numbers: $55\,416$ training, $5\,858$ validation and $393$ test reactions over $4\,874$ EC numbers, with a median of four training reactions per EC number.

EnzymeMap is used at EC level~3 with 200 classes. Its own splits hold out whole EC classes, on which a classifier can only score zero, so we hold out reactions instead: $10$~\% of the reactions of every class with at least three reactions go to the test set, drawn once with a fixed seed and shared by all runs, which gives $41\,816$ training and $4\,715$ test reactions.

\paragraph{From a marking to a product.} As previously discussed, a slack token on an atom can be a hydrogen, a lone pair, or a charge, without the marking saying which. This means that decoding a marking into molecules requires a rule set. These rules only use conservation arguments, are chosen on training reactions, and run after the token game has stopped. They are:

\begin{enumerate}
    \item An atom whose free valence decreased first loses hydrogens and becomes an onium centre such as ammonium, once there are none left. Onium carrying a hydrogen are deprotonated.
    \item S, Se and P, and I take up tokens in steps of two without a charge, while the valence expansion of S=O and P=O, and S, Se, P, and the halogens return tokens in steps of two.
    \item In an aromatic ring, surpluses, and deficits at a carbon are settled by moving a hydrogen from or to a ring nitrogen where the ring has one. These are the tautomer shifts that are hidden by bond order $\frac{3}{2}$.
    \item Total charge is conserved. This means that for every new cation, an atom that would have gained a hydrogen becomes an anion instead. Examples are N-oxides, nitro groups, and the halide counter-ions of quaternary salts.
    \item Molecules that were touched by a firing sequence and that hold two or more heavy atoms are written in neutral form. This is by convention, unless their charges pair up within the molecule as, for example, in a nitro group. Single-atom ions stay charged.
\end{enumerate}

Note that the net does not record where a hydrogen comes from, so a fragment that leaves is capped with hydrogen. The record of an atom shows bond orders, hydrogens, and formal charge as $\sum_j b_{ij}+h_i-q_i$. This is the $P$-invariant $x^{(i)\top}m$ of the valence net minus the fixed capacity $c_i$ and, for cations, minus $q_i$. In Schneider50k, this count is the same for a product atom and the precursor atom it maps to in $99.85$ \% of the product atoms. This is the conservation law the net enforces (Proposition~\ref{prop:conservation}), and it is holding in real reactions. The $0.15$ \% that are lost are the price of the guarantee. Furthermore, with these rules, the firing vectors reproduce the product for $99.3$ \% of $5\,000$ training reactions in Schneider50k and for $96.6$ \% of all $40\,000$ test reactions in USPTO-MIT.

\paragraph{The mapping metrics.} On the Golden set, we follow \citet{Lin2021}. A mapping is counts as correct if its condensed graph of reaction is isomorphic to that of the curated mapping. Where the canonical atom orders of the output of RXNMapper and of our molecules differ, the molecules are matched as graphs without stereochemistry.

\paragraph{The forward-prediction metric.} A prediction is assumed correct, if the recorded major product (the largest recorded molecule, canonical SMILES without stereochemistry, which is not contained in USPTO-MIT) is among the molecules touched by a firing and every recorded product molecule occurs in the final marking. Molecules, such as by-products and spectators omitted in the products of a data set are naturally ignored, as in other graph-edit models. All $40\,000$ test reactions of USPTO-MIT, including six that could not be parsed by RDKit count for the denominator.

\section{General nets}
\label{sec:general_nets}
Here we introduce a general net with continuous-time semantics that is separate from chemistry. It gives the layers of NPF when the rate law stays local, their guarantees and a test on measured metabolic fluxes against PGNN~\citep{AdemovicTahirovic2025}.

\subsection{Continuous-time semantics and a local rate law}
\label{sec:continuous_time}
On general nets the rate law of NPF is local, meaning that for a transition $t$ with attributes $a_t$, we set

\begin{equation}
    \lambda_{\theta,t}(m) = \exp\left( \beta_\theta(a_t) + \sum_{p \in {^\bullet}t} g_\theta(m_p, e_p, \mathrm{Pre}(p,t)) \right) \cdot \min\{1, \mathrm{enabl}(t,m)\},
\end{equation}

with two small multilayer perceptrons $\beta_\theta$ and $g_\theta$ learning a log rate constant from the attributes of the transition and a log factor per input place. They are shared across all transitions and nets. Reading only the input places, the law is local. It vanishes when an input place is empty, as Proposition~\ref{prop:enabling_rule} requires the laws. The enabling factor is continuous, as i fluid nets. The product form is that of generalised mass action, which is $g=\operatorname{Pre}(p,t)\log m_p$ up to the enabling factor.

\subsection{Forward semantics, the token-game layer}
\label{sec:forward_semantics}
This is the continuous version of the token game described in Section~\ref{sec:chemical_reactions}. Every transition fires at once, in a fractional amount, compared to the chemistry that fires one at a time. One round with step $h$ computes the demands $d_t=h\lambda_{\theta,t}(m)$, shares out the tokens of each place, limits each transition by its scarcest input, and finally applies the state equation,

\begin{equation}
    \label{eq:continuous_state_equation}
    x_p = \frac{1}{m_p} \sum_{t \in p^\bullet} \mathrm{Pre}(p,t) d_t, \quad v_t = d_t \min_{p \in {^\bullet}t} \varphi(x_p), \quad m^+ = m + Cv, \quad \varphi(x) = \frac{1 - e^{-x}}{x},
\end{equation}

with $\varphi(0)=1$, $\varphi(x_p):=0$ when $m_p=0$ and some consumer has $d_t>0$, and with $\min_{p\in^\bullet t}\varphi(x_p):=1$ when $^\bullet t=\emptyset$, so a source transition fires its full demand. A place with load $x_p$ offers $m_p(1-e^{-x_p}$ tokens that is shared in proportion to the demands. A forward model is $K$ rounds with tied weights, making depth time, and $h=1/K$. For any demands, the layer conserves and stays non-negative.

\begin{proposition}[Non-negativity of the token-game layer]
    For all $d\geq0$ and $m\geq0$, the marking $m^+$ of Equation~\ref{eq:continuous_state_equation} satisfies $m^+\geq0$, $m^+_p>0$ whenever $m_p>0$, and $x^\top m^+=x^\top m$ for every $P$-invariant $x$.
\end{proposition}

\begin{proof}
    Let  $\mathrm{out}_p=\sum_t \operatorname{Pre}(p,t)v_t$ be the consumption at $p$. If $m_p>0$, every consumer $t\in p^\bullet$ has $v_t\leq d_t\varphi(x_p)$ because $p\in^\bullet t$ takes part in the minimum, therefore $\mathrm{out}_p\leq\varphi(x_p)\sum_t\operatorname{}{Pre}(p,t)d_t=\varphi(x_p)x_p m_p=m_p(1-e^{-x_p})<m_p$. If $m_p=0$, then either $\sum_t\mathrm{Pre}(p,t)d_t=0$, so that $d_t=0$ for all consumers, or the convention $\varphi(x_p)=0$ applies. In both cases, $\mathrm{out}_p=0$. Therefore, $m_p^+=m_p-\mathrm{out}_p+\sum_t\mathrm{Post}(p,t)v_t\geq m_p e^{-x_p}$, which is non-negative, and positive if $m_p>0$. Finally, $x^\top m^+=x^\top m+x^\top Cv=x^\top m$ for every $P$-invariant $x$.
\end{proof}

Therefore, each round is an enabled, continuous step of the net, which is how the layer stays a net for all weights.

\subsection{Inverse semantics, the firing counts between two markings}
Between two markings of a single run, NPF integrates the rate law over the elapsed time along a token game from $m_A$, whose end is pinned to $m_B$. This gives counts $\Tilde{\sigma}$ and a bounded learned correction that stands in for conditioning the run on its end. A map $\Pi$ then imposes the state equation $\Hat{\sigma}=\Pi(\Tilde{\sigma})\in \mathcal{S}=\{\sigma:C\sigma=m_B-m_A\}$, as the I-projection, the point of $\mathcal{S}$ closes in Kullback-Leibler divergence to $\Tilde{\sigma}$ made positive~\citep{Csiszar1975}.

Every output is positive for all weights, meeting the state equation up to the tolerance of the solve, $0.006$ to $0.019$ mmol/gDW/h in Table~\ref{tab:metabolism}. Up to this tolerance, the error has no component in $\mathrm{im}C^\top$ and learning only has to resolve the part along the $T$-invariants, which are the cycles that cancel. Message passing cannot substitute for this layer.

\begin{proposition}[Locality lower bound]
    On the path net $p_0\rightarrow t_1\rightarrow p_1\dots t_L\rightarrow p_L$ without $T$-invariants and $\sigma^*_k=-\sum_{j<k}(m_B-m_A)_j$., let $f$ be any estimator whose output at $t_k$ depends only on the part of the input within distance $r$ of $t_k$ in the incidence graph. If $p_0$ and $p_L$ are farther than $r$ from $t_k$, then for every $\delta>0$ are two valid inputs on which $f_{t_k}$ coincides while $\sigma^*_k$ differs by $\delta$. Therefore, $\operatorname{sup}|f_{t_k}|\geq\delta/2$ for every $\delta$ that the token range admits.
\end{proposition}

\begin{proof}
    The incidence matrix of the path net has the columns $\mathbf{1}_{p_k}-\mathbf{1}_{p_{k-1}}$ that are linearly independent. This means that $\ker C=0$ and $\sigma^*$ is determined by $m_B-m_A$. Summing the state equation over $p_0,\dots,p_{k-1}$ gives $\sigma_k^*=-\sum_{j<k}(m_B-m_A)_j$. Let $(m_A, m_B)$ be a valid input with a firing sequence $s$, and let $m'_A=m_A+\delta\mathbf{1}_{p_0}$ and $m'_B=m_B+\delta\mathbf{1}_{p_L}$ with $\delta>0$, a positive integer for discrete nets. As an input, this is valid as well.  Firing $t_1,\dots,t_L$ from $m^\prime_A$ in amount $\delta$ each, moves $\delta$ tokens from $p_0$ to $t_L$ and is enabled step by step, and then firing $s$, which is enabled because enabling is monotone in the marking. Its count vector is $\sigma^*+\delta\mathbb{1}$. The two inputs only differ at $p_0$ and $p_L$, which are farther than $r$ from $t_k$, resulting in $f_{t_k}$ taking the same value on both, while the targets differ by $\delta$. One of the two errors is therefore at least $\delta/2$
\end{proof}

\subsection{Results on metabolic fluxes}
\label{sec:metabolic}
We cast metabolism as a Petri net whose markings are measured. With known stoichiometry, the uptake, secretion, and growth of a culture fix how metabolites change. The internal fluxes, measured by \textsuperscript{13}C labelling, are the firing counts to be inferred. This tests the inverse semantics on real data, against PGNN and generic regressors. We show that the net makes the fluxes balance and lets NPF transfer to new studies and carbon sources where PGNN fails, the new studies only with unmeasured by-products open, while generic regressors projected onto the net are as accurate.

\paragraph{Data and net.} Five studies estimated the central fluxes of \emph{E.~coli} by \textsuperscript{13}C metabolic flux analysis~\citep{Ishii2007,HaverkornvanRijsewijk2011,Long2019,Holm2010,Gerosa2015}, and $159$ of their cultures are kept. The net is the core model with the reactions of iJO1366~\citep{Orth2010b,Orth2011} that the studies measure or need, the Entner--Doudoroff pathway, cytochrome bo3, and the uptake of galactose, gluconate and glycerol. Places are metabolites, one biomass place, and one maintenance place. A reversible reaction is a forward and a backward transition. The two markings of a culture are one hour apart, their difference is the measured uptake, secretion, and growth. The fluxes are the firing counts of each hour. Exchanges that a study did not measure stay open as sinks, except for acetaldehyde, which no source model exchanges. We use the same nets for every method.

\paragraph{Closed by-products.} We repeat the comparison with every unmeasured carbon by-product closed (Table~\mbox{\ref{tab:metabolism_closed}}). This makes the untrained state equation worse, $0.404$ against $0.305$ within studies, and NPF inherits part of it. Its fluxes still balance, and on new carbon sources it stays ahead of PGNN ($0.387$ against $0.979$, $p=0.03$, all $6$ cultures), but within studies PGNN is ahead ($0.120$ against $0.170$, $p<0.001$), and on a new study the two are level ($0.245$ against $0.256$, $p=0.21$). A projected random forest is the most accurate on a new study and on new carbon sources ($0.154$ and $0.333$), and within studies the most accurate method whose fluxes balance ($0.117$), behind ridge regression ($0.095$).

\begin{table}[h]
\caption{Relative error of the fluxes of held-out cultures of \emph{E.~coli}, mean $\pm$ sample standard deviation over three seeds for the learned models, within the studies, on a held-out study and on other carbon sources, and the balance residual $\lVert C\sigma-(m_B-m_A)\rVert_1$ within the studies, in mmol/gDW/h.}
\label{tab:metabolism}
\centering\footnotesize\setlength{\tabcolsep}{4pt}
\begin{tabular}{@{}lrrrr@{}}
\toprule
method & within studies & new study & new carbon source & balance \\
\midrule
NPF, GMA rate law & 0.114\,$\pm$\,0.002 & 0.153\,$\pm$\,0.013 & 0.385\,$\pm$\,0.002 & 0.006 \\
NPF, rate law described in Section~\ref{sec:continuous_time} & 0.149\,$\pm$\,0.007 & 0.157\,$\pm$\,0.019 & 0.363\,$\pm$\,0.023 & 0.019 \\
PGNN, GMA rate law & 0.119\,$\pm$\,0.008 & 0.260\,$\pm$\,0.019 & 0.798\,$\pm$\,0.053 & 2\,216 \\
PGNN & 0.144\,$\pm$\,0.023 & 0.280\,$\pm$\,0.019 & 0.858\,$\pm$\,0.070 & 1\,009 \\
PGNN+ with the state equation, GMA & 0.211\,$\pm$\,0.017 & 0.208\,$\pm$\,0.072 & 0.399\,$\pm$\,0.025 & 0.008 \\
random forest, projected onto the net & 0.095 & 0.161 & 0.350 & 0.000 \\
random forest & 0.096 & 0.170 & 0.860 & 0.944 \\
ridge regression, projected onto the net & 0.130 & 0.180 & 0.655 & 0.000 \\
ridge regression & 0.095 & 0.207 & 164 & 172 \\
flux sampling & 0.196 & 0.194 & 0.355 & 0.000 \\
pFBA & 0.326 & 0.323 & 0.475 & 0.000 \\
state equation alone & 0.305 & 0.302 & 0.380 & 0.001 \\
\bottomrule
\end{tabular}
\end{table}

\begin{table}[h]
\caption{The comparison of Table~\ref{tab:metabolism} with every carbon by-product that a study did not measure closed, carbon dioxide the only free carbon exchange. Mean $\pm$ sample standard deviation over three seeds for the learned models. NPF with the rate law of Appendix~\ref{sec:continuous_time} and PGNN without read arcs were not run under this rule.}
\label{tab:metabolism_closed}
\centering\footnotesize\setlength{\tabcolsep}{4pt}
\begin{tabular}{@{}lrrrr@{}}
\toprule
method & within studies & new study & new carbon source & balance \\
\midrule
NPF, GMA rate law & 0.170\,$\pm$\,0.013 & 0.245\,$\pm$\,0.085 & 0.387\,$\pm$\,0.020 & 0.011 \\
PGNN, GMA rate law & 0.120\,$\pm$\,0.003 & 0.256\,$\pm$\,0.020 & 0.979\,$\pm$\,0.438 & 2\,052 \\
PGNN+ with the state equation, GMA & 0.265\,$\pm$\,0.088 & 0.355\,$\pm$\,0.141 & 0.348\,$\pm$\,0.016 & 2.512 \\
random forest, projected onto the net & 0.117 & 0.154 & 0.333 & 0.000 \\
random forest & 0.112 & 0.175 & 0.869 & 1.314 \\
ridge regression, projected onto the net & 0.144 & 0.176 & 0.649 & 0.000 \\
ridge regression & 0.095 & 0.207 & 164 & 172 \\
flux sampling & 0.226 & 0.223 & 0.379 & 0.000 \\
pFBA & 0.579 & 0.577 & 0.603 & 0.000 \\
state equation alone & 0.404 & 0.400 & 0.406 & 0.001 \\
\bottomrule
\end{tabular}
\end{table}

\newpage
\section{Additional results}
\subsection{The minimum firing vector}
\label{sec:minimum_firing_vector}
\paragraph{Cost.} Each mapping $\pi$ yields a firing vector $\sigma{\pi}$ by the state equation (see Figure~\ref{fig:firingvector}), and the mapper scores $\pi$ on two levels. The first level counts the bonds formed and broken, which are the bond places that $\sigma(\pi)$ takes from empty to marked or vice versa. The second level counts the changes of bond order on kept bonds, in half bonds as an aromatic bond has an order of $\frac{3}{2}$, $2\sum_{\mathrm{kept}}|\Delta b_{ij}|$, and the changes of formal charge, $2\sum_i|\Delta q_i|$. The second level only decides between mappings that tie in the first level. This is illustrated in Figure~\ref{fig:firingvector}, where $\pi_1$ forms one bond and breaks one, and $\pi_2$ forms and breaks three bonds each. Now note that hydrogens do not have atom numbers, which we handle in the following way: For the cost, in the first level, a hydrogen on carbon gets a bond place of its own, so that a formed or broken C-H bond counts like any other bond, but not again in the second level. Most C-H bonds do not exchange under ordinary conditions, so we assume that a change of hydrogen on carbon is a bond that was formed or broken by the reaction. A hydrogen on atoms other than carbon, mostly O, N, or S, is not counted, as proton transfer with a protic solvent or with an acid or a base exchanges the hydrogen quickly, so its number on a product atom reflects the acid-base form in which the data set writes the product rather than a bond change. Finally, a map does not need to use every precursor atom. The atoms it leaves out go into by-products that are missing on the product side, such as the Cl of HCL in Figure~\ref{fig:firingvector}a. In Schneider50k, this happens in $99.4$ \% of reactions. A bond from a product atom to such an atom counts as broken, and a bond between two of them counts as unchanged.

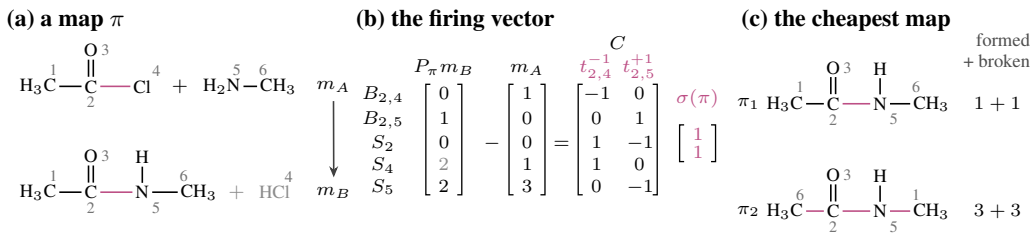
\begin{figure}[h]
    \centering
    \begin{tikzpicture}[
    font=\footnotesize, >={Stealth[length=1.5mm, width=1.2mm]},
    atom/.style={inner sep=0.7pt, font=\scriptsize},
    leaving/.style={atom, text=black!45},
    num/.style={font=\tiny, text=black!55, inner sep=0.3pt},
    bond/.style={line width=0.6pt, shorten <=0.3pt, shorten >=0.3pt},
    change/.style={bond, draw=magenta!75!black},
    flow/.style={->, line width=0.6pt, black!75},
    head/.style={font=\footnotesize\bfseries, anchor=west, inner sep=0pt},
    entry/.style={font=\scriptsize, inner sep=1pt, text width=3.7mm, align=center},
    bracket/.style={line width=0.5pt}]

  \newcommand{\carbonyl}[2]{%
    \draw[bond] ([xshift=-0.9pt]#1.north) -- ([xshift=-0.9pt]#2.south);
    \draw[bond] ([xshift=0.9pt]#1.north) -- ([xshift=0.9pt]#2.south);}

  \node[head] at (0, 4.72) {(a) a map $\pi$};

  \node[atom, label={[num]45:1}] (a1) at (0.42, 3.85) {H$_3$C};
  \node[atom, label={[num]270:2}] (a2) at (1.12, 3.85) {C};
  \node[atom, label={[num]0:3}] (a3) at (1.12, 4.33) {O};
  \node[atom, label={[num]45:4}] (a4) at (1.8, 3.85) {Cl};
  \node[atom] at (2.3, 3.85) {$+$};
  \node[atom, label={[num]45:5}] (a5) at (2.85, 3.85) {H$_2$N};
  \node[atom, label={[num]135:6}] (a6) at (3.6, 3.85) {CH$_3$};
  \draw[bond] (a1) -- (a2);
  \carbonyl{a2}{a3}
  \draw[change] (a2) -- (a4);
  \draw[bond] (a5) -- (a6);

  \node[atom, label={[num]45:1}] (b1) at (0.42, 2.45) {H$_3$C};
  \node[atom, label={[num]270:2}] (b2) at (1.12, 2.45) {C};
  \node[atom, label={[num]0:3}] (b3) at (1.12, 2.93) {O};
  \node[atom, label={[num]315:5}] (b5) at (1.8, 2.45) {N};
  \node[atom] (bh) at (1.8, 2.91) {H};
  \node[atom, label={[num]135:6}] (b6) at (2.55, 2.45) {CH$_3$};
  \node[leaving] at (3.05, 2.45) {$+$};
  \node[leaving, label={[num]45:4}] at (3.55, 2.45) {HCl};
  \draw[bond] (b1) -- (b2);
  \carbonyl{b2}{b3}
  \draw[change] (b2) -- (b5);
  \draw[bond] (b5) -- (bh);
  \draw[bond] (b5) -- (b6);

  \node[font=\scriptsize, anchor=west] (ma) at (3.98, 3.85) {$m_A$};
  \node[font=\scriptsize, anchor=west] (mb) at (3.98, 2.45) {$m_B$};
  \draw[flow] (ma) -- (mb);

  \node[head] at (4.62, 4.72) {(b) the firing vector};

  \matrix (M) [matrix of math nodes, anchor=north west, row sep=0.4pt, column sep=1.6pt,
               nodes={font=\scriptsize, inner sep=1pt}] at (4.55, 4.43) {
    & P_\pi m_B & & m_A & & |[text=magenta!75!black]| t^{-1}_{2,4} & |[text=magenta!75!black]| t^{+1}_{2,5} \\
    B_{2,4} & |[entry]| 0 & & |[entry]| 1 & & |[entry]| -1 & |[entry]| 0 \\
    B_{2,5} & |[entry]| 1 & & |[entry]| 0 & & |[entry]| 0 & |[entry]| 1 \\
    S_2 & |[entry]| 0 & - & |[entry]| 0 & = & |[entry]| 1 & |[entry]| -1 \\
    S_4 & |[entry, text=black!45]| 2 & & |[entry]| 1 & & |[entry]| 1 & |[entry]| 0 \\
    S_5 & |[entry]| 2 & & |[entry]| 3 & & |[entry]| 0 & |[entry]| -1 \\
  };

  \matrix (V) [matrix of math nodes, anchor=west, row sep=0.4pt,
               nodes={font=\scriptsize, inner sep=1pt, text=magenta!75!black}] at ($(M-4-7.east)+(0.2, 0)$) {
    |[entry]| 1 \\
    |[entry]| 1 \\
  };
  \node[font=\scriptsize, text=magenta!75!black, anchor=south] at ($(V-1-1.north)+(0, 0.1)$) {$\sigma(\pi)$};
  \node[font=\scriptsize, anchor=base] at ($(M-1-6.base)!0.5!(M-1-7.base)+(0, 0.3)$) {$C$};

  \foreach \l/\r/\t/\b in {M-2-2/M-2-2/M-2-2/M-6-2, M-2-4/M-2-4/M-2-4/M-6-4, M-2-6/M-2-7/M-2-6/M-6-6,
                           V-1-1/V-1-1/V-1-1/V-2-1} {
    \draw[bracket] ([xshift=1.4pt]\l.west |- \t.north) -- ([xshift=-0.3pt]\l.west |- \t.north)
                   -- ([xshift=-0.3pt]\l.west |- \b.south) -- ([xshift=1.4pt]\l.west |- \b.south);
    \draw[bracket] ([xshift=-1.4pt]\r.east |- \t.north) -- ([xshift=0.3pt]\r.east |- \t.north)
                   -- ([xshift=0.3pt]\r.east |- \b.south) -- ([xshift=-1.4pt]\r.east |- \b.south);
  }

  \node[head] at (9.72, 4.72) {(c) the cheapest map};
  \node[font=\scriptsize, text=black!70, inner sep=0pt, align=right, anchor=south east] at (13.55, 4.1) {formed\\ + broken};

  \node[font=\scriptsize] at (9.82, 3.62) {$\pi_1$};
  \node[atom, label={[num]45:1}] (c1) at (10.3, 3.62) {H$_3$C};
  \node[atom, label={[num]270:2}] (c2) at (10.94, 3.62) {C};
  \node[atom, label={[num]0:3}] (c3) at (10.94, 4.1) {O};
  \node[atom, label={[num]315:5}] (c5) at (11.56, 3.62) {N};
  \node[atom] (ch) at (11.56, 4.08) {H};
  \node[atom, label={[num]135:6}] (c6) at (12.24, 3.62) {CH$_3$};
  \draw[bond] (c1) -- (c2);
  \carbonyl{c2}{c3}
  \draw[change] (c2) -- (c5);
  \draw[bond] (c5) -- (ch);
  \draw[bond] (c5) -- (c6);
  \node[font=\scriptsize, anchor=east] at (13.55, 3.62) {$1+1$};

  \node[font=\scriptsize] at (9.82, 2.2) {$\pi_2$};
  \node[atom, label={[num]45:6}] (d1) at (10.3, 2.2) {H$_3$C};
  \node[atom, label={[num]270:2}] (d2) at (10.94, 2.2) {C};
  \node[atom, label={[num]0:3}] (d3) at (10.94, 2.68) {O};
  \node[atom, label={[num]315:5}] (d5) at (11.56, 2.2) {N};
  \node[atom] (dh) at (11.56, 2.66) {H};
  \node[atom, label={[num]135:1}] (d6) at (12.24, 2.2) {CH$_3$};
  \draw[change] (d1) -- (d2);
  \carbonyl{d2}{d3}
  \draw[change] (d2) -- (d5);
  \draw[bond] (d5) -- (dh);
  \draw[change] (d5) -- (d6);
  \node[font=\scriptsize, anchor=east] at (13.55, 2.2) {$3+3$};

  \pgfresetboundingbox
  \useasboundingbox (-0.03, 1.72) rectangle ++(13.6, 3.18);
\end{tikzpicture}%
    \caption{The firing vector. (a) A map $\pi$ seats each product atom on the precursor atom with its number, and grey atoms leave unrecorded. (b) The state equation gives $\sigma(\pi)$, one firing per changed bond place, and the grey entry follows as $m_A+C\sigma$. The H that moves from N to Cl passes from $S_5$ to $S_4$, while cycles in $\ker C$, such as the C=O of the tetrahedral intermediate, and the order of the firings stay unseen. (c) The mapper ranks maps by the bonds formed and broken, then by bond orders and charges, and returns $\hat\sigma=\sigma(\pi_1)$.}
    \label{fig:firingvector}
\end{figure}

This cost is not fitted. The levels are compared in order rather than weighted against each other, with every count having a weight of one. How the cost is defined is a design decision we made based on experimenting on 200 random reactions on the Golden set, with the remaining $1\,560$ held out (Table~\ref{tab:golden}).

\paragraph{Finding the cheapest map.} We use branch and bound. As the number of possible maps grows factorially with the number of atoms, the mapper doesn't try them all. It places the product atoms on precursor atoms of the same atom type one by one, and, after each placement, bounds the cost of every way to finish the map from below. A branch whose bound exceeds the cost the search aims for is dropped. This is branch and bound. The bound compares neighbourhoods. The bonds among placed atoms are already fixed, so a product atom placed on a precursor atom with different neighbours causes at least one bond change. By matching all atoms that are not yet placed to their candidates at once and halving the sum (as a bond has two atoms) bounds the bonds formed and broken in the rest of the map, while a second assignment bounds the changes of bond order and charge. This assignment also rules out single candidates that cannot reach the target. An atom with one candidate left is placed at once, and precursor atoms a symmetry of the molecule exchanges are tried only once. The target starts at the bound of the empty map and rises by one bond at a time, making the first complete map found the cheapest. A search then collects all other maps with the same cost for the third level discussed in the next paragraph. A search that finishes proves that no cheaper maps exist.

\begin{table}[h]
\centering
\footnotesize
\caption{Atom mapping on the Golden set. Scored by the full condensed graph of reaction, for every level of the cost, on the 200 reactions of the random choice, the 1\,560 held out, and all 1\,760 usable reactions with a 95~\% bootstrap interval. Of the held out reactions, 112 are only mapped correctly by the best variant of our mapper, and 66 only by RXNMapper. Proved is the share of reactions where the search finishes within 60~s.}
\label{tab:golden}
\begin{tabular}{@{}lrrrrr@{}}
\toprule
cost & choice & held out & all & interval & proved \\
\midrule
bonds made and broken only & 83.0 & 81.03 & 81.25 & 79.4 to 83.1 & 99.94 \\
+ second level, all hydrogen & 83.5 & 82.44 & 82.56 & 80.8 to 84.3 & 99.89 \\
+ hydrogen on heteroatoms not counted & 85.0 & 83.40 & 83.58 & 81.8 to 85.3 & 99.89 \\
+ C--H as bond places (the cost) & 86.0 & 83.65 & 83.92 & 82.2 to 85.6 & 100.00 \\
+ third level (the mapper) & 91.0 & 88.53 & 88.81 & 87.3 to 90.2 & 100.00 \\
RXNMapper & 85.5 & 85.58 & 85.57 & 83.9 to 87.2 & -- \\
\bottomrule
\end{tabular}
\end{table}

\paragraph{Introducing a third level.} After implementing tie-breaking levels 1 and 2, a comparison with RXNMapper showed that most of our errors were still ties between maps of equally low cost. Therefore, we chose three further terms against the curated maps of the 200 development reactions, in the following order. (1) The sum of the absolute changes of oxidation state of the product carbons, as most reactions oxidise or reduce few carbons. We chose this because a step that is not a redox step moves no electrons between carbons. (2) The number of bond places that switch between an aromatic carbon and a heteroatom, as substitution there requires activation. (3) Has the number of bond places switched between an electron sink and an atom other than a carbon. A sink is a non-aromatic atom with a double or triple bond to a more electronegative atom, which can pass an incoming electron pair to that atom and take it back. This count enters with a negative sign, so the mapper prefers a nucleophile that adds to a sink while a leaving group leaves it.

\begin{table}[h]
\centering
\footnotesize
\caption{The Golden set by the number of places that the curated map changes, high numbers may be one-pot or multi-step reactions. Accuracy is given in \%, the reactions that only our mapper or only RXNMapper map correctly with the exact sign test, and the share of curated maps that are not minimal, as they change more places than the cheapest map. Our mapper is ahead in every bin. Curated maps that are not minimal become more common the more places change, as a mechanism with more steps is likelier to move atoms along a costlier path, and a bond that forms and breaks again is invisible for every map.}
\label{tab:bins}
\begin{tabular}{@{}lrrrrrr@{}}
\toprule
places & reactions & NPF & RXNMapper & only one & $p$ & not minimal \\
\midrule
1 or 2 & 772 & 95.1 & 92.7 & 38\,/\,20 & 0.025 & 0.4 \\
3 or 4 & 658 & 87.8 & 84.8 & 46\,/\,26 & 0.024 & 3.6 \\
5 or 6 & 212 & 82.1 & 74.5 & 30\,/\,14 & 0.023 & 3.8 \\
7 or more & 118 & 65.3 & 62.7 & 16\,/\,13 & 0.711 & 12.7 \\
\bottomrule
\end{tabular}
\end{table}

\paragraph{The open net for missing reactants.} In $4.75$~\% of reactions in the Golden data set and $0.85$~\% in Schneider50k, the products have more atoms of a given atom type than precursors. This means that no injection exists. For these reactions, the net opens with two source transitions, one for a potentially missing precursor molecule and one for a single atom of an atom type at a deficit in the products. Each costs one unit on each level. This allows us to map all 88 Golden and 426 Schneider50k reactions that previously failed.

\begin{table}[h]
\centering
\footnotesize
\caption{Atom mapping on SynRXN, our mapper against each mapper from the original publication. Reactions that only our mapper maps correctly / only the other mapper, and the exact McNemar $p$. The accuracies are in Table~\ref{tab:mapping}. Only 5 of the 382 Recon3D references are clean injections, the others reuse atom-map numbers, which caps every bijective mapper there, and 257 of the 273 of E.~coli are.}
\label{tab:mcnemar}
\begin{tabular}{@{}lrrrr@{}}
\toprule
set & RDT & RXNMapper & LocalMapper & GraphormerMapper \\
\midrule
Golden & 175\,/\,49, $<$\,0.001 & 109\,/\,69, 0.003 & 119\,/\,108, 0.51 & 122\,/\,120, 0.95 \\
NatComm & 41\,/\,9, $<$\,0.001 & 24\,/\,9, 0.01 & 21\,/\,31, 0.21 & 21\,/\,32, 0.17 \\
USPTO-3k & 78\,/\,40, $<$\,0.001 & 3\,/\,45, $<$\,0.001 & 4\,/\,173, $<$\,0.001 & 17\,/\,106, $<$\,0.001 \\
Recon3D & 15\,/\,4, 0.02 & 35\,/\,0, $<$\,0.001 & 30\,/\,3, $<$\,0.001 & 90\,/\,2, $<$\,0.001 \\
E.\ coli & 17\,/\,31, 0.06 & 30\,/\,29, 1.00 & 22\,/\,14, 0.24 & 90\,/\,6, $<$\,0.001 \\
\bottomrule
\end{tabular}
\end{table}

\paragraph{EnzymeMap mapping}Results of our mapping system in comparison to RXNMapping of the EnzymeMap dataset reactions Table~\ref{tab:enzymemap_mapping}. 

\begin{table}[ht]
\centering
\footnotesize
\caption{Atom mapping of enzymatic reactions on EnzymeMap ($41\,510$ reactions), against its validated maps. Correct
reaction centre under the validator of SynRXN, as in Table~\ref{tab:mapping} in \%. Only one: reactions mapped correctly by our mapper only / by RXNMapper only,
with the exact sign test.}
\label{tab:enzymemap_mapping}
\begin{tabular}{@{}lrrrr@{}}
\toprule
criterion & NPF, minimum firing vector & RXNMapper & only one & $p$ \\
\midrule
reaction centre (SynRXN) & \textbf{88.70} & 77.92 & $5\,213$ / $741$ & $<0.001$ \\

\bottomrule
\end{tabular}
\end{table}

\subsection{Reaction classification}
\label{sec:app_reaction_classification}

Table~\ref{tab:schneider} shows the results on the published split of Schneider50k. Table~\ref{tab:appendix_classification} shows the results of experiments testing whether the two classifiers of the net, namely the state-equation readout and the firing vector of the mapper, are more accurate than counterparts of the same size without the net, a generic readout that pools each side, a generic head on the same maps, and DRFP by~\citet{Probst2022reaction} as a strong baseline. We ask where the Petri net helps.

\paragraph{Against the established baseline.} The firing vector is, up to hydrogens and charges, the condensed graph of reaction counted as firings, so a D-MPNN on the condensed graph of reaction~\citep{Heid2021machine}, which reads the same maps is the established model closest to it. We ran it with Chemprop~2.3.1 by \citet{Heid2023} on the same training, validation, and test reactions, label draws, and size split, with the default settings of Chemprop and nothing tuned, in three arms (Table~\ref{tab:appendix_classification}), namely, the default reaction mode with 357K parameters, the balanced mode, in which atoms that leave do not count as broken bonds, and the default mode with widths of 850 instead of 300 (2.4M parameters). With all training reactions and on the size split, Chemprop keeps the epoch of the lowest validation loss. With a label budget, no validation reaction enters, and Chemprop keeps the epoch of the lowest training loss. With its default 50 epochs, a label budget gives only 200 optimiser steps at 250 labels, where the training loss stayed far above zero. After these runs had been seen, we added runs with the optimiser steps of the default run on all training reactions. The CGR D-MPNN is ahead of DRFP with all training reactions ($97.61\pm0.13$~\% at width 850) and behind the firing vector ($98.80\pm0.04$~\%, $p<0.001$). The gap widens with fewer labels. Its best arms reach $89.18\pm0.70$~\% on $1\,000$ and $40.53\pm2.66$~\% on 250 training reactions, against $95.07\pm0.65$ and $81.19\pm1.27$~\% for the firing vector (both with $p<0.001$) and $86.62$ and $65.03$~\% for DRFP. On larger molecules its best arm, the balanced mode, reaches $77.74\pm3.00$~\%, level with the firing vector with the gate ($81.30\pm1.49$~\%, $p=0.17$) and with DRFP ($78.43$~\%).

\begin{table}[h]
\centering
\footnotesize
\setlength{\tabcolsep}{4pt}
\newcommand{\mpm}[2]{$#1 \pm #2$}
\newcommand{\bmpm}[2]{{\boldmath$#1 \pm #2$}}
\caption{Reaction classification on Schneider50k, published split with 200 training and 800 test reactions per class. Mean $\pm$ sample standard deviation over $n$ seeds. Best in \textbf{bold}, second \underline{underlined}.}
\label{tab:schneider}
\begin{tabular}{@{}lrrr@{}}
\toprule
method & $n$ & accuracy & macro-F1 \\
\midrule
DRFP + MLP                       & 3 & \mpm{95.52}{0.04} & \mpm{95.47}{0.04} \\
generic readout                  & 3 & \mpm{97.74}{0.01} & \mpm{97.72}{0.02} \\
generic head, firing vector      & 5 & \underline{\mpm{98.64}{0.09}} & \underline{\mpm{98.63}{0.09}} \\
NPF, state-equation readout      & 5 & \mpm{97.68}{0.09} & \mpm{97.67}{0.10} \\
NPF, firing vector of the mapper & 5 & \bmpm{98.80}{0.04} & \bmpm{98.80}{0.04} \\
\bottomrule
\end{tabular}
\end{table}

\begin{table}[h]
\centering
\footnotesize
\caption{Classification on Schneider50k, accuracy on the $39\,994$ test reactions of the published split, mean $\pm$ sample standard deviation over 5 seeds for the readout, the firing vector and its generic head with all labels and 3 otherwise. With $1\,000$ or 250 labelled reactions, 20 or 5 per class, a stratified subsample drawn with the seed of the run, alone trains on these reactions only and + firing task also on the maps of the unlabelled rest. Larger molecules trained on reactions up to the median size and tested on the largest quarter. Firing task trained with the auxiliary task on the training reactions without a label (Appendix~\ref{sec:implementation_details}). A dash where the task does not apply, as DRFP and the CGR D-MPNN have no net to attach it to, and the firing vector already holds what it predicts. The CGR D-MPNN is Chemprop~2.3.1 on the condensed graph of reaction with the maps of the mapper, 5 seeds with all training reactions and 3 otherwise, and with a label budget and matched steps it trains for the optimiser steps of its default run on all training reactions.}
\label{tab:appendix_classification}
\setlength{\tabcolsep}{3pt}
\resizebox{\linewidth}{!}{%
\begin{tabular}{@{}lrrrrrr@{}}
\toprule
& & & \multicolumn{2}{c}{1\,000 labelled} & \multicolumn{2}{c}{250 labelled} \\
\cmidrule(lr){4-5}\cmidrule(l){6-7}
model & full training set & larger molecules & alone & + firing task & alone & + firing task \\
\midrule
DRFP + MLP & 95.52\,$\pm$\,0.04 & 78.43\,$\pm$\,0.63 & 86.62\,$\pm$\,0.77 & -- & 65.03\,$\pm$\,1.58 & -- \\
generic readout & 97.74\,$\pm$\,0.01 & 71.76\,$\pm$\,1.08 & 78.59\,$\pm$\,1.27 & 90.02\,$\pm$\,0.16 & 32.14\,$\pm$\,1.64 & 52.64\,$\pm$\,1.33 \\
NPF, state-equation readout without gate & 97.47\,$\pm$\,0.14 & 67.43\,$\pm$\,1.03 & 91.26\,$\pm$\,0.27 & 94.17\,$\pm$\,0.32 & 61.19\,$\pm$\,4.24 & 81.65\,$\pm$\,3.21 \\
NPF, state-equation readout & 97.68\,$\pm$\,0.09 & 71.55\,$\pm$\,3.65 & 91.57\,$\pm$\,0.64 & 94.76\,$\pm$\,0.20 & 57.01\,$\pm$\,3.19 & 86.77\,$\pm$\,1.26 \\
generic head on the firing vector & 98.64\,$\pm$\,0.09 & 82.83\,$\pm$\,1.03 & 93.83\,$\pm$\,0.51 & -- & 48.99\,$\pm$\,1.31 & -- \\
NPF, firing vector of the mapper & 98.80\,$\pm$\,0.04 & 76.07\,$\pm$\,3.75 & 95.07\,$\pm$\,0.65 & -- & 81.19\,$\pm$\,1.27 & -- \\
\quad with the participation gate & 98.73\,$\pm$\,0.10 & 81.30\,$\pm$\,1.49 & 95.04\,$\pm$\,0.55 & -- & 79.77\,$\pm$\,2.26 & -- \\
\midrule
CGR D-MPNN, Chemprop defaults (357K) & 96.83\,$\pm$\,0.12 & 63.43\,$\pm$\,1.26 & 61.25\,$\pm$\,3.61 & -- & 12.54\,$\pm$\,2.32 & -- \\
\quad balanced reaction mode & 97.01\,$\pm$\,0.11 & 77.74\,$\pm$\,3.00 & 76.83\,$\pm$\,5.10 & -- & 9.14\,$\pm$\,2.90 & -- \\
\quad width 850 (2.4M) & 97.61\,$\pm$\,0.13 & 72.35\,$\pm$\,1.02 & 83.33\,$\pm$\,1.28 & -- & 25.29\,$\pm$\,3.61 & -- \\
\quad defaults, matched steps & -- & -- & 83.39\,$\pm$\,2.19 & -- & 30.33\,$\pm$\,1.06 & -- \\
\quad balanced, matched steps & -- & -- & 89.18\,$\pm$\,0.70 & -- & 40.53\,$\pm$\,2.66 & -- \\
\quad width 850, matched steps & -- & -- & 89.02\,$\pm$\,1.50 & -- & 39.56\,$\pm$\,2.65 & -- \\
\bottomrule
\end{tabular}}
\end{table}

\paragraph{Locality of the readout.} The contribution of a matched atom through the atom terms of the class readout falls from $22.2$ at the re-typed bond to $0.001$ four bonds away, reaches below $10^{-5}$ for $99.8$~\% of the atoms four bonds away. This is predicted by Proposition~\ref{prop:cancellation} for the atom terms $R=K=3$. The exceptions lie near changes of attributes that are not bond changes.

\begin{table}[ht]
\centering
\footnotesize
\setlength{\tabcolsep}{4pt}
\caption{EC classification on EnzymeMap at EC level~3 (200 classes; $41\,816$ training and $4\,715$ test reactions,
$10$~\% of every class held out), accuracy and macro-F1 in \%, mean $\pm$ sample standard deviation over three seeds. The same classifier reads the firing vector of each source of atom maps: without maps, the state-equation readout reads the two sides of the reaction. $p$ from Welch's test against the classifier without maps.}
\label{tab:enzymemap_classification}
\begin{tabular}{@{}lrrrr@{}}
\toprule
atom maps & accuracy & macro-F1 & $p$ (accuracy) & $p$ (macro-F1) \\
\midrule
none, state-equation readout        & \pmsd{90.36}{0.48} & \pmsd{66.80}{0.77} & -- & -- \\
RXNMapper                           & \pmsd{90.47}{1.00} & \pmsd{67.23}{1.96} & 0.87 & 0.75 \\
minimum firing vector (the mapper)  & \pmsd{91.04}{0.74} & \pmsd{67.51}{1.54} & 0.26 & 0.53 \\
\midrule
EnzymeMap's curated maps            & \textbf{\pmsd{91.81}{0.12}} & \textbf{\pmsd{69.63}{0.58}} & 0.03 & 0.008 \\
\bottomrule
\end{tabular}
\end{table}

\subsection{Forward prediction}
\label{sec:app_forward_prediction}

\paragraph{Targets of the forward model.} The forward model is trained on the set of tied minimum firing vectors of the mapper instead of the atom maps provided by the data set. On 10~\% of USPTO-MIT these targets are ahead of the supplied maps by $1.2$ points, $79.72\pm0.03$ against $78.57\pm0.04$~\% greedy top-1 ($p<0.001$, Table~\ref{tab:targets}), and by $1.1$ points on the major product, $79.48$ against $78.37$~\%. One vector of each set reaches $79.15\pm0.26$~\%, between the two ($p=0.06$ against each).

\begin{table}[h]
\centering
\footnotesize
\caption{Forward prediction on the nested 1, 10 and 25~\% subsets of the USPTO-MIT training file ($4\,090$, $40\,900$ and $102\,259$ lines), top-1 and top-5 in \% of all $40\,000$ test reactions for the major product, the largest molecule of a prediction against the recorded major product, from beam search. NPF: mean $\pm$ sample standard deviation over three seeds, 1.6M parameters at width 128 and 9.8M at width 256. The Molecular Transformer (MT) is trained from scratch for $5\,000$, $50\,000$ and $75\,000$ steps, one run each.}
\label{tab:subsets}
\setlength{\tabcolsep}{3pt}
\resizebox{\linewidth}{!}{%
\begin{tabular}{@{}lrrrrrr@{}}
\toprule
& \multicolumn{2}{c}{1~\%} & \multicolumn{2}{c}{10~\%} & \multicolumn{2}{c}{25~\%} \\
\cmidrule(lr){2-3}\cmidrule(lr){4-5}\cmidrule(l){6-7}
model & top-1 & top-5 & top-1 & top-5 & top-1 & top-5 \\
\midrule
MT & 0.16 & 0.37 & 74.16 & 84.19 & 82.39 & 90.37 \\
NPF, supplied maps & 67.24\,$\pm$\,0.24 & 80.36\,$\pm$\,0.36 & 78.85\,$\pm$\,0.04 & 88.73\,$\pm$\,0.08 & -- & -- \\
NPF, targets of the mapper & 67.44\,$\pm$\,0.19 & 80.45\,$\pm$\,0.18 & 78.83\,$\pm$\,0.14 & 88.78\,$\pm$\,0.16 & 81.73\,$\pm$\,0.13 & 90.64\,$\pm$\,0.06 \\
NPF width 256, supplied maps & -- & -- & -- & -- & 82.95\,$\pm$\,0.26 & 91.04\,$\pm$\,0.12 \\
NPF width 256, targets of the mapper & -- & -- & -- & -- & 83.23\,$\pm$\,0.20 & 91.00\,$\pm$\,0.07 \\
\bottomrule
\end{tabular}
}
\end{table}

\begin{table}[h]
\centering
\footnotesize
\caption{Targets of the forward model on 10~\% of USPTO-MIT, greedy and beam top-1 in \% on all $40\,000$ test reactions under the rule of Appendix~\ref{sec:implementation_details}, mean $\pm$ sample standard deviation over three seeds. Training reactions are those of the subset with a target that decodes to the product.}
\label{tab:targets}
\begin{tabular}{@{}lrrr@{}}
\toprule
targets & training reactions & greedy top-1 & beam top-1 \\
\midrule
atom maps supplied in the data set & $39\,936$ & 78.57\,$\pm$\,0.04 & 79.05\,$\pm$\,0.04 \\
mapper, set of tied vectors & $39\,381$ & 79.72\,$\pm$\,0.03 & 80.08\,$\pm$\,0.05 \\
mapper, one vector & $39\,381$ & 79.15\,$\pm$\,0.26 & 79.54\,$\pm$\,0.25 \\
\bottomrule
\end{tabular}
\end{table}

\begin{table}[h]
\centering
\footnotesize
\caption{Forward prediction on our split of Schneider50k, $4\,667$ test reactions and 3 seeds. Greedy top-1 against the number of training reactions with a target, beam top-1 and top-5 with all of them, and the share of valence-valid predictions.}
\label{tab:appfwd}
\setlength{\tabcolsep}{3pt}
\resizebox{\linewidth}{!}{%
\begin{tabular}{@{}lrrrrrr@{}}
\toprule
& \multicolumn{3}{c}{greedy top-1} & \multicolumn{2}{c}{beam} & \\
\cmidrule(lr){2-4}\cmidrule(lr){5-6}
model & $1\,995$ & $7\,973$ & $37\,220$ & top-1 & top-5 &
valid \\
\midrule
NPF token game & 84.31\,$\pm$\,0.71 & 90.34\,$\pm$\,0.34 & 93.34\,$\pm$\,0.28 & 93.39\,$\pm$\,0.18 & 98.08\,$\pm$\,0.07 & 100.00\,$\pm$\,0.00 \\
\quad without the enabling mask & 84.19\,$\pm$\,0.18 & 89.82\,$\pm$\,0.64 & 93.06\,$\pm$\,0.11 & 93.13\,$\pm$\,0.12 & 98.07\,$\pm$\,0.04 & 99.98\,$\pm$\,0.02 \\
one-shot, greedy valence repair & -- & -- & 81.60\,$\pm$\,0.50 & -- & -- & 100.00\,$\pm$\,0.00 \\
one-shot, generic & 66.28\,$\pm$\,0.51 & 75.02\,$\pm$\,0.53 & 79.72\,$\pm$\,0.22 & -- & -- & 97.21\,$\pm$\,0.30 \\
\bottomrule
\end{tabular}}
\end{table}

\begin{table}[ht]
\centering
\footnotesize
\setlength{\tabcolsep}{4pt}
\caption{Validity of the top-1 predictions on USPTO-480K ($39\,994$ test reactions that RDKit parses), one seed each.
\emph{Valence rule}: share of predictions whose markings respect the enabling rule of the valence net. \emph{Valid molecules}: share of predictions whose products RDKit sanitises, with the number of invalid predictions. \emph{Dropped fragments}: share of predictions in which a touched molecule could not be decoded and was left out. Invalid products of the aromatic models come from the half-token tolerance on aromatic bonds.}
\label{tab:forward_validity}
\begin{tabular}{@{}lrrrr@{}}
\toprule
NPF model & valence rule & valid molecules & invalid & dropped fragments \\
\midrule
aromatic, recorded maps                     & 100      & 99.875 & 50 & --    \\
aromatic, targets of the net                & 100      & 99.897 & 41 & 0.130 \\
aromatic, targets of the net, no enabling   & 99.987   & 99.890 & 44 & 0.138 \\
\bottomrule
\end{tabular}
\end{table}

\subsection{Reaction mechanisms}
\label{sec:app_reaction_mechanisms}

The arrow net scores $88.85$~\% top-1. Compared to the electron net, the difference comes entirely from steps that move single electrons. On the $98.21$~\% of test steps that move only pairs, the two nets match in scores, while the remaining $2\,905$ steps that move single electrons cannot be expressed by the arrow net. The electron net predicts $91.08$~\% correctly. Most of these steps are binding or release of phosphine at palladium, whose dative bond is written as a bond between two neutral atoms by FlowER.

\begin{table}[!htbp]
\centering
\footnotesize
\setlength{\tabcolsep}{4pt}
\caption{Validity and conservation of the predictions on the FlowER test set, in \%. Valid SMILES of the top-1 prediction, and the share of predictions that conserve heavy atoms, and heavy atoms, protons and electrons together. For NPF and FlowER, conservation holds by construction, as every firing keeps the P-invariants of the net (Proposition~\ref{prop:conservation}) and FlowER's electron changes sum to zero. For NPF we report mean sample standard deviation over three seeds, all $162\,002$ steps. Baselines from the text of \citet{Joung2025}, single runs, where FlowER's validity is given as approximately 95~\%.}
\label{tab:mechanism_validity}
\begin{tabular}{@{}lrrr@{}}
\toprule
method & valid SMILES & heavy atoms & atoms, protons, electrons \\
\midrule
Molecular Transformer & 70.2 & 39.1 & 33.0 \\
Graph2SMILES          & 76.3 & 30.7 & 17.2 \\
Graph2SMILES+H        & 78.8 & 27.7 & 19.0 \\
FlowER                & 95 & by construction & by construction \\
\midrule
NPF, arrow net        & 100 & by construction & by construction \\
NPF, electron net     & 100 & by construction & by construction \\
\bottomrule
\end{tabular}
\end{table}

\subsection{What does each Petri component contribute?}
\label{sec:app_load_bearing}
The Petri net buys guarantees for every value of the weights, and accuracy where training data are scarce (Table~\ref{tab:loadbearing}). Enabling keeps every prediction valence-valid and the pure readout is unmoved by atoms far from the reaction, where the generic counterparts either break a budget or change their prediction. The mask itself does not add accuracy. Playing the net one bond change at a time is more accurate at every training size, the readout and the firing vector gain most on small training sets, and neither gains on molecules larger than those in the training set. With the participation gate, they match the generic counterparts.

\paragraph{Firing one bond change at a time.} On our split of Schneider50k, the token game predicts the product for $93.34$~\% of the test reactions and a one-shot model with the same encoder, which labels all bond places at once and trains on the first vector of each target set, for $79.72$~\%. This gap grows with fewer training reactions. A valence repair, dropping a bond formation or adding a bond break until every atom is obeying the enabling rule, thereby making every output of the one-shot model valid only lifts it to $81.60$~\%. We can therefore say that the accuracy comes from sequential decoding, with each bond change being chosen at the marking the previous ones produced. This is an empirical effect without a proposition. Of course, sequential graph edits are not new, see ~\citet{do2019} and ~\citet{Sacha2021}, but here, they need neither a supervised order nor atom maps. Note that this comparison does not isolate conservation, as the one-shot model shares the firing form.

\begin{table}[h]
\centering
\footnotesize
\setlength{\tabcolsep}{4pt}
\caption{Each Petri component against a counterpart of equal size without it on Schneider50k, with $4\,667$ forward test reactions, 3 seeds for forward prediction, and 3 to 5 seeds for classification. Mean $\pm$ sample standard deviation where shown, $p$ from a Welch test, and a dash where a proposition guarantees the value. The rows of Proposition~\ref{prop:cancellation} are the pure readout, $w\equiv1$, for which it holds. The subsamples of $1\,000$ and 250 training reactions, 20 and 5 per class, are stratified and drawn with the seed of the run, the same for all models of that seed. \emph{Italic}, not load-bearing, as the Petri component is not significantly better, $p\ge0.05$, or is worse.}
\label{tab:loadbearing}
\setlength{\tabcolsep}{2pt}
\resizebox{\textwidth}{!}{
\begin{tabular}{@{}llrrr@{}}
\toprule
component & criterion & Petri & generic & $p$ \\
\midrule
one bond change at a time & top-1, all training reactions & 93.34\,$\pm$\,0.28 & 79.72\,$\pm$\,0.22 & $<$\,0.001 \\
\quad same & top-1, 1\,995 training reactions & 84.31\,$\pm$\,0.71 & 66.28\,$\pm$\,0.51 & $<$\,0.001 \\
enabling (Prop.~\ref{prop:validity}) & valence-valid products & 100 (all weights) & 97.21\,$\pm$\,0.30 & -- \\
\emph{enabling mask in the game} & \emph{top-1} & \emph{93.34\,$\pm$\,0.28} & \emph{93.06\,$\pm$\,0.11} & \emph{0.22} \\
readout (Prop.~\ref{prop:cancellation}) & \% changed by a far methyl & 0.00 & 0.27 & -- \\
\quad same & accuracy, 1\,000 / 250 training reactions & 91.26 / 61.19 & 78.59 / 32.14 & 0.002 / 0.003 \\
\emph{\quad same} & \emph{all training reactions / larger molecules} & \emph{97.47 / 67.43} & \emph{97.74 / 71.76} & \emph{0.08 / 0.007} \\
\emph{\quad with the gate} & \emph{all training reactions / larger molecules} & \emph{97.68 / 71.55} & \emph{97.74 / 71.76} & \emph{0.21 / 0.93} \\
firing vector & accuracy, 250 training reactions & 81.19 & 48.99 & $<$\,0.001 \\
\quad against a CGR D-MPNN of equal size & accuracy, 1\,000 / 250 training reactions & 95.07 / 81.19 & 89.02 / 39.56 & 0.01 / $<$\,0.001 \\
\emph{\quad same} & \emph{accuracy, 1\,000 training reactions} & \emph{95.07} & \emph{93.83} & \emph{0.06} \\
\quad same & all training reactions & 98.80\,$\pm$\,0.04 & 98.64\,$\pm$\,0.09 & 0.01 \\
\emph{\quad same} & \emph{larger molecules} & \emph{76.07} & \emph{82.83} & \emph{0.08} \\
\emph{\quad with the gate} & \emph{larger molecules} & \emph{81.30} & \emph{82.83} & \emph{0.23} \\
\bottomrule
\end{tabular}
}
\end{table}

\paragraph{Enabling.} With untrained weights on a set of 256 test reactions, none of the markings that the token game visits breaks the rule by more than half a token. As untrained weights do not know chemistry, the net and not the training keeps the products valid, while training alone does not avoid every violation. Without masks, the one-shot model exceeds the valence budget by more than half a token in $2.97$~\% of the predictions, the token game for none. Without the mask, the token game is as accurate and valence-valid in $99.98$~\% of predictions, meaning that the mask turns a learned regularity into a guarantee at no cost in accuracy because the rule removes orders but no product. On $5\,000$ random USPTO-MIT test reactions, the Schneider50k models without the mask break a valence budget in $0.16\,\pm\,0.07$~\% of the predictions. The models with the mask in none at the same accuracy. The learned regularity therefore weakens where the data change, while the guarantee holds.

\paragraph{Readout and firing vector.} Consider a methyl group attached to both sides of $4\,000$ test reactions, 9 or more bonds from every change. It changes no prediction of the pure readout, $0.12$~\% with a participation gate, and $0.27$~\% for a generic readout that pools each side. On the full training set, the pure readout is level with its generic counterpart ($p=0.08$). However, it trains it on test reactions larger than those in training ($64.73$ against $71.76$~\%, $p=0.007$, with the gate bringing it level to $71.55$~\%. On the full training set, the firing vector is ahead of its generic head ($p=0.001$) and not significantly behind it on larger molecules ($p=0.08$). However, both pay off on a stratified subsample of 250 training reactions, where the pure readout reaches $61.19$~\% against $32.14$~\% and the firing vector $81.19$~\% against $48.99$~\% for a generic mead on the same maps. On a larger subsample of $1\,000$, the readout stays ahead of its generic counterpart ($p=0.002$), while the lead of the firing vector is not significant with $95.07$ against $93.83$~\% ($p=0.06$).

\end{document}